\documentclass[11pt]{article}
\pdfoutput=1
\usepackage{enumerate}
\usepackage{pdfsync}
\usepackage[OT1]{fontenc}

\usepackage[usenames]{color}
\usepackage{smile}
\usepackage[colorlinks,
            linkcolor=red,
            anchorcolor=blue,
            citecolor=blue
            ]{hyperref}
\usepackage{fullpage}
\usepackage{hyperref}
\usepackage[protrusion=true,expansion=true]{microtype}
\usepackage{pbox}
\usepackage{setspace}
\usepackage{tabularx}
\usepackage{float}
\usepackage{wrapfig,lipsum}
\usepackage{enumitem}
\usepackage{microtype}
\usepackage{graphicx}
\usepackage{subfigure}
\usepackage{enumerate}
\usepackage{enumitem}
\usepackage{pgfplots}
\usetikzlibrary{arrows,shapes,snakes,automata,backgrounds,petri}
\usepackage{booktabs} 

\newcommand{\mbb}{\mathbb}
\newcommand{\mcal}{\mathcal}

\newcommand{\regret}{\mathbf{Regret}}
\newcommand{\stoch}{\text{stoch}}
\newcommand{\oline}{\overline}

\allowdisplaybreaks

\usepackage{enumitem}

\makeatletter
\newcommand*{\rom}[1]{\expandafter\@slowromancap\romannumeral #1@}
\makeatother
\title{\huge Linear Contextual Bandits with Adversarial Corruptions}

\author
{
Heyang Zhao\thanks{IIIS, Tsinghua University, Beijing, CN; e-mail: {\tt zhaoheya18@mails.tsinghua.edu.cn}}
	~~~and~~~
	Dongruo Zhou\thanks{Department of Computer Science, University of California, Los Angeles, CA 90095, USA; e-mail: {\tt drzhou@cs.ucla.edu}}
	~~~and~~~
	Quanquan Gu\thanks{Department of Computer Science, University of California, Los Angeles, CA 90095, USA; e-mail: {\tt qgu@cs.ucla.edu}}
}

\begin{document}

\date{}
\maketitle

\begin{abstract}
We study the linear contextual bandit problem in the presence of adversarial corruption, where the interaction between the player and a possibly infinite decision set is contaminated by an adversary that can corrupt the reward up to a corruption level $C$ measured by the sum of the largest alteration on rewards in each round. 
We present a variance-aware algorithm that is adaptive to the level of adversarial contamination $C$. The key algorithmic design includes (1) a multi-level partition scheme of the observed data, (2) a cascade of confidence sets that are adaptive to the level of the corruption, and (3) a variance-aware confidence set construction that can take advantage of low-variance reward. We further prove that the regret of the proposed algorithm is $\tilde{O}(C^2d\sqrt{\sum_{t = 1}^T \sigma_t^2} + C^2R\sqrt{dT})$, where $d$ is the dimension of context vectors, $T$ is the number of rounds, $R$ is the range of noise and $\sigma_t^2,t=1\ldots,T$ are the variances of instantaneous reward. We also prove a gap-dependent regret bound for the proposed algorithm, which is instance-dependent and thus leads to better performance on good practical instances. To the best of our knowledge, this is the first variance-aware corruption-robust algorithm for contextual bandits. Experiments on synthetic data corroborate our theory. 
\end{abstract}

\section{Introduction}

Multi-armed bandit algorithms are widely applied in online advertising \citep{li2010contextual}, clinical trials \citep{villar2015multi}, recommendation system \citep{deshpande2012linear} and many other real-world tasks. In the model of multi-armed bandits, the algorithm needs to decide which action (or arm) to take (or pull) at each round and receives a reward for the chosen action. In the stochastic setting, the reward is subject to a fixed but unknown distribution for each action. In reality, however, these rewards can easily be ``corrupted'' by some malicious users. A typical example is click fraud \citep{lykouris2018stochastic}
, where botnets simulate the legitimate users clicking on an ad to fool the 
recommendation systems. This motivates the studies of bandit algorithms that are robust to adversarial corruptions. For example, \cite{lykouris2018stochastic} introduced a bandit model in which an adversary could corrupt the stochastic reward generated by an arm pull. They proposed an algorithm and showed that the regret of the algorithm 
degrades smoothly as the amount of corruption diminishes. \cite{pmlr-v99-gupta19a} proposed an alternative algorithm which gives a significant improvement in regret. 

While the algorithms that are robust to the corruptions have been studied in the setting of multi-armed bandits in a number of prior works, they are still understudied in the setting of linear contextual bandits. The linear contextual bandit problem can be regarded as an extension of the multi-armed bandit problem to linear optimization, in order to tackle an unfixed and possibly infinite set of feasible actions. 
There is a large body of literature on efficient algorithms for linear contextual bandits without corruptions \citep{abe2003reinforcement, auer2002using,  dani2008stochastic,li2010contextual,rusmevichientong2010linearly, chu2011contextual,AbbasiYadkori2011ImprovedAF, li2019nearly}, to mention a few.
The significance of this setting lies in the fact that in real world applications, arms often come with contextual information that can be utilized to facilitate arm selection 
\citep{li2010contextual, deshpande2012linear,jhalani2016linear}. 
Linear contextual bandits with adversarial corruptions is an arguably more challenging setting since most of the previous corruption-robust algorithms are based on the idea of action elimination \citep{lykouris2018stochastic, pmlr-v99-gupta19a, pmlr-v130-bogunovic21a}, which is not applicable to the contextual bandit settings where the decision set is time varying and possibly infinite at each round. \citet{garcelon2020adversarial} showed that a malicious agent can force a linear contextual bandit algorithm to take any desired action $T - o(T)$ times over $T$ rounds, while applying adversarial corruptions to rewards with a cumulative cost that only grow logarithmically. This poses a big challenge for designing corruption-robust algorithms for linear contextual bandits.

In this paper, we make a first attempt to study a linear contextual bandit model where an adversary can corrupt the rewards up to a corruption level $C$, which is defined as the the sum of biggest alteration on the reward the adversary makes in each round.  
We propose a linear contextual bandit algorithm that is robust to reward corruption, dubbed multi-level optimism-in-the-face-of-uncertainty weighted learning (Multi-level weighted OFUL). More specifically, our algorithm consists of the following novel techniques: (1) We design a multi-level partition scheme and adopt the idea of \emph{sub-sampling} to perform a robust estimation of the model parameters; (2) We maintain a cascade of \emph{candidate confidence sets} corresponding to different corruption level (which is unknown) and randomly select a confidence set at each round to take the action; and (3) We design confidence sets that depend on the \emph{variances of rewards}, which lead to a potentially tighter regret bound. 

Our contributions are summarized as follows:
\begin{itemize}[leftmargin = *, nosep]
    \item We propose a variance-aware algorithm which is adaptive to the amount of adversarial corruptions $C$. To the best of our knowledge, it is the first algorithm for the setting of linear contextual bandits with adversarial corruptions which does not rely on the finite number of actions and other additional assumptions. 
    \item We prove that the regret of our algorithm is $\tilde{O}\left(C^2d\sqrt{\sum_{t = 1}^T \sigma_t^2} + C^2R\sqrt{dT}\right)$, where $d$ is the dimension of context vectors, $T$ is the number of rounds, $R$ is the range of noise and $\sigma_t^2,t=1\ldots,T$ are the variances of instantaneous reward. Our regret upper bound has a multiplicative dependence on $C^2$ 
    which indicates that our algorithm achieves a sub-linear regret when the corruption level satisfies $C = o(T^{1 / 4})$. 
    \item We also derive a gap-dependent regret bound $ \tilde{O}\left(\frac{1}{\Delta}\cdot C^2R^2d + \frac{1}{\Delta}\cdot d^2C^2 \max_{t \in[T]} \sigma_t^2\right)$ for our proposed algorithm, which is instance-dependent and thus leads to a better performance on good instances.
\end{itemize}

Concurrent to our work, \citet{ding2021robust} also studied linear contextual bandits under adversarial attacks where the adversary can attack on the rewards. However, a careful examination of their proof finds that their proof is flawed, and their seemingly better regret is questionable\footnote{Their application of Theorem 1 in the proof of Theorem 2 seems untenable since Theorem 1 essentially provides a bound on $\EE\left[\sum_{t \in [T]} x_{t, *}^\top \theta - X_t(C^*)^\top \theta | J_{j_s} = C^* \ \forall s\in [T]\right]$ using their notation, which is not necessarily greater than or equal to Quantity(A) in their proof.}.

\paragraph{Notation.} 
We use lower case letters to denote scalars, and use lower and upper case bold face letters to denote vectors and matrices respectively. We denote by $[n]$ the set $\{1,\dots, n\}$. For a vector $\xb\in \RR^d$ and matrix $\bSigma\in \RR^{d\times d}$, a positive semi-definite matrix, we denote by $\|\xb\|_2$ the vector's Euclidean norm and define $\|\xb\|_{\bSigma}=\sqrt{\xb^\top\bSigma\xb}$. For two positive sequences $\{a_n\}$ and $\{b_n\}$ with $n=1,2,\dots$, 
we write $a_n=O(b_n)$ if there exists an absolute constant $C>0$ such that $a_n\leq Cb_n$ holds for all $n\ge 1$ and write $a_n=\Omega(b_n)$ if there exists an absolute constant $C>0$ such that $a_n\geq Cb_n$ holds for all $n\ge 1$. We use $\tilde O(\cdot)$ to further hide the polylogarithmic factors. We use $\ind(\cdot)$ to denote the indicator function. 

\section{Related Work}

\textbf{Bandits with Adversarial Rewards:} 
There is a large body of literature on the problems of multi-armed bandits in the adversarial setting (See \citet{auer2002nonstochastic, bubeck2012regret} and references therein). 
Many recent efforts in this area aim to design algorithms that achieve desirable regret bound in both stochastic multi-armed bandits and adversarial bandits simultaneously, known as ``the best of both worlds'' 
guarantees \citep{pmlr-v23-bubeck12b, pmlr-v32-seldinb14, pmlr-v49-auer16, pmlr-v65-seldin17a, pmlr-v89-zimmert19a}. These works mainly focus on achieving sublinear regret in the worst case and the case where there is no adversary. As a result, these algorithms are either not robust to instances with moderate amount of corruptions, or suffer from restrictive assumptions on adversarial corruptions (e.g., \cite{pmlr-v32-seldinb14} and \cite{pmlr-v89-zimmert19a} assumed that the adversarial corruptions do not reduce the gap by more than a constant factor at any point of time). Different from the above line of research, \cite{lykouris2018stochastic} studied a variant of classic multi-armed bandit model in the ``middle ground'', where each pull of an arm generates a stochastic reward that may be contaminated by an adversary before it is revealed to the player. In their work, the corruption level $C$ is defined as $C = \sum_{t} \max_a |r^t(a) - r_\cS^t(a)|$ where $r_\cS^t(a)$ is the stochastic reward of arm $a$ and $r^t(a)$ is the corrupted reward of arm $a$ at round $t$. 
They proposed an algorithm that is adaptive to the unknown corruption level, which achieves an $O(K^{1.5} C \sqrt{T})$ regret bound. \cite{pmlr-v99-gupta19a} proposed an improved algorithm that can achieve a regret bound with only additive dependence on $C$. 
On the flip side, many research efforts have also been devoted into designing adversarial attacks that cause standard bandit algorithms to fail \citep{jun2018adversarial, liu2019data, lykouris2018stochastic, garcelon2020adversarial}. 

\noindent \textbf{Stochastic Linear Bandits with Corruptions: }
\cite{Li2019StochasticLO} studied stochastic linear bandits with adversarial corruptions and achieved $\tilde{O}( d^{5/2} C/\Delta + d^6/\Delta^2)$ regret bound where $d$ is the dimension of the context vectors, $\Delta$ is the gap between the rewards of the best and the second best actions in the decision set $\cD$. \cite{pmlr-v130-bogunovic21a} also studied corrupted linear bandits with a fixed decision set of $k$ arms and obtained $\tilde{O}(\sqrt{dT} + C d^{3/2} + C^2)$ regret upper bound. Recently, \cite{lee2021achieving} considered corrupted linear bandits with a finite and fixed decision set and achieved a regret of $\tilde{O}(d\sqrt{T} + C)$. 
While both \citet{lee2021achieving} and \citet{Li2019StochasticLO} focused on corrupted linear stochastic bandits, \cite{lee2021achieving} used a slightly different definition of regret and adopted a strong assumption on corruptions that at each round $t$, the corruptions on rewards are linear in the actions. 

\noindent \textbf{Linear Contextual Bandits with Corruptions:} 
\cite{pmlr-v130-bogunovic21a} studied linear contextual bandits with adversarial corruptions and considered the setting under the assumption that context vectors undergo small random perturbations, which is previously introduced by \cite{Kannan2018ASA}. Besides the additional assumption, another major distinction in \cite{pmlr-v130-bogunovic21a} is that the number of actions $k$ is finite and the regret bound depends on $k$ in the contextual setting with unknown corruption level $C$. 
\cite{neu2020efficient} studied misspecified linear contextual bandits with a finite decision set (i.e., $K$ actions) and proved an $\tilde{O}((Kd)^{\frac{1}{3}} T^{\frac{2}{3}}) + \epsilon \cdot \sqrt{d}T$ regret bound for their proposed algorithm. 
\cite{kapoor2019corruption} considered the corrupted linear contextual bandits under an assumption on corruptions that for any prefix, at most an
$\eta$ fraction of the rounds are corrupted. 


\section{Preliminaries} \label{section:3}


We will introduce our model and some basic concepts in this section. 

\noindent \textbf{Corrupted linear contextual bandits. } We consider the the linear contextual bandits model studied in \citet{AbbasiYadkori2011ImprovedAF} under the same corruption studied by \cite{lykouris2018stochastic}. In detail, distinctive from the linear contextual bandits \citet{AbbasiYadkori2011ImprovedAF}, the interaction between the agent and the environment is now contaminated by an adversary. The protocol between the agent and the adversary at each round $t \in [T]$ is described as follows: 
\fbox{\begin{minipage}{\linewidth - 3mm}
\begin{enumerate}[leftmargin = *, nosep]
    \item \label{model:step:1}At the beginning of round $t$, the environment generates an arbitrary decision set $\cD_t \subseteq \RR^d$ where each element represents a feasible action that can be selected by the agent. 
    \item \label{model:step:2} The environment generates stochastic reward function $r_t'(\ab) = \langle \ab, \bmu^* \rangle + \epsilon_t(\ab)$ together with an upper bound on the standard variance of $\epsilon_t(\ab)$, i.e., $\sigma_t(\ab)$ for all $\ab \in \cD_t$.  
    \item \label{model:step:3} The adversary observes $\cD_t$, $r_t'(\ab)$, $\sigma_t(\ab)$ for all $\ab \in \cD_t$ and decides a corrupted reward function $r_t$ defined over $\cD_t$.  
    \item \label{model:step:4} The agent observes $\cD_t$ and selects $\ab_t \in \cD_t$. 
    \item \label{model:step:5} The adversary observes $\ab_t$ and then returns $r_t(\ab_t)$ and $\sigma_t(\ab_t)$. 
    \item \label{model:step:6} The agent observes $r_t(\ab_t), \sigma_t(\ab_t)$. 
\end{enumerate}\end{minipage}}

\vspace{0.5em}

Let $\cF_t$ be the $\sigma$-algebra generated by $\cD_{1:t}, \ab_{1:t - 1}, \epsilon_{1:t - 1}, r_{1:t - 1}$ and $\sigma_{1:t - 1}$. 

At step \ref{model:step:2}, $\bmu^*$ is a hidden vector unknown to the agent which can be observed by the adversary at the beginning. 
We assume that for all $t \ge 1$ and all $\ab \in \cD_t$, $\|\ab\|_2 \le A$, $|\langle \ab, \bmu^* \rangle| \le 1$ and $\|\bmu^*\|_2 \le B$ almost surely. $\epsilon_t(\ab)$ can be any form of random noise as long as it satisfies 
\begin{align}
    \forall t \ge 1, &\forall \ab \in \cD_t, |\epsilon_t(\ab)| \le R, \quad\EE[\epsilon_t(\ab)|\cF_t] = 0, \EE[\epsilon_t^2(\ab)|\cF_t] \le \sigma_t^2(\ab). \label{noise}
\end{align}


This assumption on $\epsilon_t$ is a variant of that in \cite{Zhou2020NearlyMO}: Here we require the noise to be generated for all $\ab \in \cD_t$ in advance before the adversary decides the corrupted reward function. 
Our assumption on noises is more general than those in \citet{Li2019StochasticLO, pmlr-v130-bogunovic21a, kapoor2019corruption} where the noises are assumed to be 1-sub-Gaussian or Gaussian. The motivation behind this assumption is that the environment may change over time in practical applications. Also, on the theoretical aspect, the setting of heteroscedastic noise is more general and can be extended to the Markov decision processes (MDPs) in reinforcement learning \citep{Zhou2020NearlyMO}, where the ``noise" is caused by the transition of states and the variance can be estimated by the agent. 

At step \ref{model:step:3}, the adversary has observed all the previous information and thus may predict which policy the agent will take at the current round. However, since the agent can take a randomized policy, the adversary may not know exactly which action the agent will take. 

\noindent \textbf{Corruption level. } We define corruption level 
\begin{align}
C = \frac{1}{R + 1} \sum_{t = 1}^T \sup_{\ab \in \cD_t} |r_t'(\ab) - r_t(\ab)|. \label{def:corruption}
\end{align} 
to indicate the level of adversarial contamination.
We say a model is $C$-corrupted if the corruption level is no larger than $C$. 

Our definition of corruption level is equivalent to the counterpart in \cite{lykouris2018stochastic} and \cite{pmlr-v99-gupta19a} where they define $C = \sum_{t = 1}^T \max_{\ab} |r_t'(\ab) - r_t(\ab)|$ in our notation of rewards. We introduce a normalization factor of $R + 1$ since the noise is of range $R$ in our model, while they assume all the rewards are in range $[0, 1]$. 

\noindent \textbf{Regret.} Since the actions selected by the agent may not be deterministic, we define the regret for this model as follows: \begin{align}
    \regret(T) = \sum_{t = 1}^T \langle \ab_t^*, \bmu^*\rangle - \EE\left[\sum_{t = 1}^T \langle \ab_t, \bmu^* \rangle\right] \label{regret}. 
\end{align}

Our definition follows from the definition in \cite{pmlr-v99-gupta19a} where the pseudo-regret (a standard metric in stochastic multi-armed bandit models of) is adopted. It is worth noting that we need to take the expectation on $\sum_{t = 1}^T r_t'(\ab_t)$ (the second term in \eqref{regret}), since a randomized policy is applied in each round. 

\noindent \textbf{Gap.} 
Let $\Delta_t$ be the gap between the rewards of the best and the second best actions in the decision set $\cD_t$ as defined in \cite{dani2008stochastic} which can be formally written as \begin{align}
    \Delta_t = \min_{\ab \in \cD_t, \ab \not\in \cA_t^*} \left(\langle \ab_t^*, \bmu^* \rangle - \langle \ab, \bmu^* \rangle\right). \label{def:gap}
\end{align} where $\cA_t^* = \argmax_{\ab \in \cD_t}\langle \ab, \bmu^* \rangle$ and $\ab_t^*$ is an arbitrary element in $\cA_t^*$. Let $\Delta$ denotes the smallest gap $\min_{t\in[T]} \Delta_t$. 

\section{Warm up: Algorithm for Known Corruption Level} \label{section:warmup}

\begin{algorithm}[t!]
    \caption{Robust weighted OFUL}\label{alg:2}
    \begin{algorithmic}[1]
        \STATE Set $\bSigma_1 \leftarrow \lambda \Ib, \bmu_1 \leftarrow \zero, \cbb_1 \leftarrow 0. $
        \FOR {$\ell = 1, \cdots, T$}
        \STATE Observe $\cD_t$. 
        \STATE Set $\cC_t$ as defined in \eqref{eq:def:ct}. 
        \STATE Select $\ab_t \leftarrow \argmax_{\ab \in \cD_t} \max_{\bmu \in \cC_t}$ and observe $r_t, \sigma_t$. 
        \STATE Set $\oline{\sigma}_t = \max\{(R + 1) / \sqrt{d}, \sigma_t\}. $
        \STATE Update estimator: $\bSigma_{t + 1} \leftarrow \bSigma_{t} + \ab_t\ab_t^\top / \oline{\sigma}_t^2$, $\cbb_{t + 1} \leftarrow \cbb_t + r_t \ab_t / \oline{\sigma}_t^2$, $\bmu_{t + 1} \leftarrow \bSigma_{t + 1}^{-1} \cbb_{t + 1}. $
        \ENDFOR
    \end{algorithmic}
\end{algorithm}


In this section, we show that if the corruption level $C$ is revealed to the agent, a robust version of weighted OFUL \citep{AbbasiYadkori2011ImprovedAF,Zhou2020NearlyMO} in Algorithm \ref{alg:2} can achieve a regret upper bound of $\tilde{O}(CRd\sqrt{T})$. 

The key idea is that we use an enlarged confidence bound to adapt to the known corruption level $C$: 
\begin{align} 
    \cC_t = \{\bmu | \|\bmu - \bmu_t \| \le \alpha_t\} \label{eq:def:ct}
\end{align} where \begin{align} 
    \alpha_t &= 8\sqrt{d\log \frac{(R + 1)^2\lambda + tA^2}{(R + 1)^2\lambda}\log(4t^2 / \delta)} \notag \\&\quad +4\sqrt{d}\log(4t^2 / \delta) +  C{\sqrt{d}} + \sqrt{\lambda}B.  \label{eq:def:alpha}
\end{align}

\begin{lemma}[Enlarged Confidence Ellipsoid. ] \label{lemma:oful1}
With probability at least $1 - \delta$, we have $\bmu^* \in \cC_t$ for all $t\ge 1$. 
\end{lemma}


According to the above lemma, we can compute such enlarged confidence ellipsoids at each round $t$ to ensure that $\bmu^*$ is still in the confidence sets under adversarial attacks if we know the corruption level in advance. 

With the enlarged confidence sets, we show that Algorithm \ref{alg:2} achieves a regret upper bound adapting to the known corruption level $C$ in the following theorem. 

\begin{theorem} \label{thm:alg:1}
    Set $\lambda = 1 / B^2$. Suppose $C = \Omega(1), R = \Omega(1)$, for all $t \ge 1$ and all $\ab \in \cD_t$, $\langle \ab, \bmu^* \rangle \in [-1, 1]$. Then with probability at least $1 - \delta$, the regret of Algorithm \ref{alg:2} is bounded as follows:
    \begin{align} 
        \regret(T) = \tilde{O}\left(C d \sqrt{\sum_{t = 1}^T \sigma_t^2} + CR \sqrt{dT}\right). 
    \end{align}
\end{theorem}


\begin{remark} \label{remark:4.3}
   Theorem \ref{thm:alg:1} suggests that when the corruption level $C$ is known, Algorithm \ref{alg:2} incurs a regret that has a linear dependence on $C$. On the other hand, if we trivially upper bound $\sigma_t$'s by $R$, the regret of Algorithm \ref{alg:2} degenerates to $ \tilde{O}\left(CdR\sqrt{T}\right)$. This suggests that the use of variance information can lead to a tighter regret bound.
\end{remark}

\begin{remark}
When $\sigma_t , R= \Omega(1)$, the regret bound in Theorem \ref{thm:alg:1} matches the regret bound of weighted OFUL proposed in \citet{Zhou2020NearlyMO} when the corruption level $C$ is a constant. 
\end{remark}


\section{Algorithm for Unknown Corruption Level} \label{section:4}

In this section, we propose an algorithm, Multi-level weighted OFUL, in Algorithm \ref{alg:1}, to tackle the corrupted linear contextual bandit problem with unknown corruption level. At the core of our algorithm is an action partition scheme to group historical selected actions and use them to select the future actions in different groups with different probabilities. Such a scheme is introduced to deal with the unknown corruption level. For simplicity, we denote $r_t(\ab_t), \sigma_t(\ab_t)$ in Section \ref{section:3} by $r_t, \sigma_t$ in our algorithm. 

\begin{algorithm}[t!]
    \caption{Multi-level weighted OFUL}\label{alg:1}
        \begin{algorithmic}[1]
            \STATE Set the largest level of confidence sets: $\ell_{\max} \leftarrow \lceil \log_2 2T \rceil$.  
            \STATE For $\ell \in [\ell_{\max}]$, set $\bSigma_{1, \ell} \leftarrow \lambda \Ib, \bmu_{1, \ell} \leftarrow \zero, \cbb_{1, \ell} \leftarrow \zero. $
            \STATE Set $\bSigma_{1} \leftarrow \lambda \Ib, \bmu_{1} \leftarrow \zero, \cbb_{1} \leftarrow \zero.$
            \FOR{$t = 1, \cdots, T$}
                \STATE Observe $\cD_t. $
                \FOR {$\ell = 1, \cdots, \ell_{\max}$} \label{line:6}
                \STATE Set $\beta_{t, \ell}$ and $\gamma_{t, \ell}$ as defined in \eqref{def:beta} and \eqref{def:gamma}. 
                \STATE $\cC_{t, \ell}' \leftarrow \left\{\bmu|\|\bmu - \bmu_t\|_{\bSigma_{t}} \le \beta_{t, \ell}\right\} \cap \left\{\bmu|\|\bmu - \bmu_{t, \ell}\|_{\bSigma_{t, \ell}} \le \gamma_{t, \ell}\right\}. $ 
                \STATE $\cC_{t, \ell} \leftarrow \begin{cases}
                \cC_{t, \ell}', & \cC_{t, \ell}' \neq \varnothing \\
                \cC_{t, \ell + 1}, & \text{otherwise}
                \end{cases}. $
                \ENDFOR \label{line:10}
                 \STATE Set $f(t) = \begin{cases}
                    \ell\quad\quad  \text{w.p.} \ 2^{-\ell} & 1 < \ell \le \ell_{\max} \\
                    1 & \text{otherwise}
                \end{cases}. $ \label{line:11}
                
            \STATE Select $\ab_{t} \leftarrow \argmax_{\ab \in \cD_t}\max_{\bmu \in \cC_{t, f(t)}}\langle \bmu, \ab \rangle $ and observe $r_t$, $\sigma_t$.  \label{line:12}
                \STATE Set $\oline{\sigma}_t = \max\{(R + 1) / \sqrt{d}, \sigma_t\}. $ \label{line:13}
                \STATE $\bSigma_{t + 1} \leftarrow \bSigma_{t} + \ab_t\ab_t^\top / \oline{\sigma}_t^2$, $\cbb_{t + 1} \leftarrow \cbb_t + r_t \ab_t / \oline{\sigma}_t^2$, \\$\bmu_{t + 1} \leftarrow \bSigma_{t + 1}^{-1} \cbb_{t + 1}. $\label{line:14} 
            \FOR {$\ell \neq f(t)$} \label{line:15}
                \STATE $\bSigma_{t + 1, \ell} \leftarrow \bSigma_{t, \ell}, \cbb_{t + 1, \ell} \leftarrow \cbb_{t, \ell}, \bmu_{t + 1, \ell} \leftarrow \bmu_{t, \ell}. $
                \ENDFOR
                \STATE $\bSigma_{t + 1, f(t)} \leftarrow \bSigma_{t, f(t)} + \ab_t\ab_t^\top / \oline{\sigma}_t^2,$ \\ $\cbb_{t + 1, f(t)} \leftarrow \cbb_{t, f(t)} + r_t \ab_t / \oline{\sigma}_t^2. $ 
                \STATE $\bmu_{t + 1, f(t)} \leftarrow \bSigma_{t + 1, f(t)}^{-1} \cbb_{t + 1, f(t)}. $
            \ENDFOR  \label{line:20}
        \end{algorithmic}
\end{algorithm}




\noindent \textbf{Action partition scheme. }To address the unknown $C$ issue, besides the original estimator $\bmu_t$ which uses all previous data, Algorithm \ref{alg:1} maintains several additional learners to learn $\bmu^*$ at different accuracy level simultaneously, and it \emph{randomly} selects one of the learners with different probabilities at each round. Such a ``parallel learning'' idea is inspired by \cite{lykouris2018stochastic}. In detail, we partition the observed data into $\ell_{\max}$ levels indexed by $[\ell_{\max}]$ and maintain $\ell_{\max}$ sub-sampled estimators $\bmu_{t, 1}, \cdots, \bmu_{t, \ell_{\max}}$. 
According to line~\ref{line:11}, the observed data in round $t$ goes into level $\ell$ with probability $2^{-\ell}$ if $1 < \ell \le \ell_{\max}$ and it goes to level 1 with probability $1 - \sum_{\ell = 2}^{\ell_{\max}} 2^{-\ell} = 1 / 2 + 2^{-\ell_{\max}}$. The intuition is that if $2^\ell \ge C$, then the corruption experienced at level $\ell$ 
\begin{align}
    \text{Corruption}_{t, \ell} = \sum_{i = 1}^t \frac{\mathds{1}(f(i) = \ell)}{R + 1} \sup_{\ab \in \cD_i} \left|r_i(\ab) - r_i'(\ab)\right| \label{def:sbcorruption}
\end{align} 
can be bounded by some quantity that is \emph{independent of} $C$. That says, the individual learners whose level is greater than $\log C$ can learn $\bmu^*$ successfully, even with the corruption. For the learners whose level is less than $\log C$, we can also control the error by controlling the probability for the agent to select them.  

\noindent \textbf{Weighted regression estimator. }
After introducing the partition scheme, we still need to deal with the varying variance (heteroscedastic) case. Similar to \citet{kirschner2018information, Zhou2020NearlyMO}, we proposed the following \emph{weighted ridge regression estimator}, which incorporates the variance information of the rewards into estimation:
\begin{align}
    \bmu_t = \argmin_{\bmu \in \RR^d} \lambda \|\bmu\|_2^2 + \sum_{i = 1}^{t - 1} [\langle \bmu, \ab_i \rangle - r_i]^2 / \oline{\sigma}_i^2. \label{global:mu}
\end{align} 
Here $\oline{\sigma}_t$ is defined as the upper bound of the true variance $\sigma_t$ in line \ref{line:13}. The closed-form solution to \eqref{global:mu} is calculated at each round in line \ref{line:14}. The use of $\oline{\sigma}_t$, as we will show later, makes our estimator statistically more efficient in the heteroscedastic case. Meanwhile, we also apply our weighted regression estimator to each individual learner, and their estimator $\bmu_{t, \ell}$ can be written as 
\begin{align}
    \argmin_{\bmu \in \RR^d} \lambda \|\bmu\|_2^2 + \sum_{i = 1}^{t - 1} \mathds{1}(f(i) = \ell)\frac{[\langle \bmu, \ab_i \rangle - r_i]^2}{\oline{\sigma}_i^2}. \label{subsample:mu} 
\end{align}
The closed-form solution to \eqref{subsample:mu} is calculated at each round in lines \ref{line:15}--\ref{line:20}.

\noindent \textbf{Final Multi-Level confidence sets. }With the estimators $\bmu_t$, $\bmu_{t, 1}, \cdots, \bmu_{t, \ell_{\max}}$ at the beginning of round $t$, we define a cascade of candidate confidence sets as in lines \ref{line:6}--\ref{line:10}, where \begin{align}
    \beta_{t, \ell} &= 8\sqrt{d\log \frac{(R + 1)^2\lambda + tA^2}{(R + 1)^2\lambda}\log(4t^2 / \delta)} +4\sqrt{d}\log(4t^2 / \delta) +  2^\ell{\sqrt{d}} + \sqrt{\lambda}B, \label{def:beta} \\
    \gamma_{t, \ell} &= 8\sqrt{d\log \frac{(R + 1)^2\lambda + tA^2}{(R + 1)^2\lambda}\log(8t^2 T / \delta)}+4\sqrt{d}\log(8t^2 T / \delta) + \oline{C}_\ell {\sqrt{d}}  + \sqrt{\lambda}B, \label{def:gamma}
\end{align} with $\oline{C}_\ell = \log(2 \ell^2 / \delta) + 3$.  
For brevity, we define \begin{align} 
\ell^* = \max\{2, \lceil \log_2 C \rceil\} \label{robustlayer} \end{align} as an important threshold in our later proof for regret bound analysis. 
Later we will prove that $\cC_{t, \ell}$ contains $\bmu^*$ for all $\ell \ge \ell^*$, $t \ge 1$ with high probability. 

Note that each candidate confidence set can be written as the intersection of two ellipsoids. The intuition behind our construction of candidate confidence sets is that we hope that $\cC_{t, \ell}$ is robust enough to handle the $2^\ell$-corrupted case, i.e., $\bmu^* \in \cC_{t, \ell}$ with high probability. To achieve this, the first ellipsoid makes use of the global information and the ``radius'' $\beta_{t, \ell}$ need to contain a factor of $2^\ell$ to tolerate a corruption level of $2^\ell$, and the second ellipsoid only makes use of the observed data assigned to level $\ell$ since this part of data only encounters $O(\log(\ell / \delta))$ corruption with probability at least $1 - \delta$ when $\ell$ is large enough as we will show later. 

\noindent \textbf{Action selection. }With the candidate confidence sets, we use line \ref{line:11} to randomly choose one confidence set and select an action based on the optimism-in-the-face-of-uncertainty (OFU) principle in line \ref{line:12}. Then we update the estimators for the $t + 1$-th round . 

\begin{remark}
Our algorithm shares a similar strategy for partitioning the observed data as the algorithm in \cite{lykouris2018stochastic}. Nevertheless, there is a major difference: \cite{lykouris2018stochastic} regard the partition scheme as a ``layer structure'', i.e., their algorithm further uses different estimators in layers of parallel learners and does action elimination layer by layer in each round. In contrast, the sub-sampled estimators in our algorithm are used independently, i.e., the selected action only relies on one of the partitions. As a result, Algorithm \ref{alg:1} does not need to do action elimination, thus is capable of handling the cases where the number of actions is huge or even infinite. 
\end{remark}

\section{Main Results}\label{sec:main results}

In this section, we present our main theorem, which establishes the regret bound for Multi-level weighted OFUL.
\begin{theorem}\label{thm:regret}
    Set $\lambda = 1/B^2$. Suppose that $C = \Omega(1)$, $R = \Omega(1)$, for all $t \ge 1$ and all $\ab \in \cD_t$, $\langle \ab, \bmu^* \rangle \in [-1, 1]$. Then with probability at least $1 - 3 \delta$, the regret of Algorithm \ref{alg:1} is bounded as follows: 
    $$\regret(T) = \tilde{O}\left(C^2d\sqrt{\sum_{t = 1}^T \sigma_t^2} + C^2 R\sqrt{dT}\right). $$ 
\end{theorem}

A few remarks are in order. 
\begin{remark}
Compared with the regret upper bound in Theorem \ref{thm:alg:1} for known corruption level case, our regret bound in unknown corruption level case suffers an extra factor of $C$, which is caused by the multi-level structure to deal with the unknown corruption level. It remains an open problem if the dependence on $C$ is tight. 
\end{remark}

\begin{remark}\label{remark:regret1}
Compared with the $\tilde O(d\sqrt{T} + C)$ result in \citet{lee2021achieving}, our result has a multiplicative quadratic dependence on $C$, which seems to be worse. Nevertheless, we want to emphasize that we focus on the linear contextual bandit setting, where the decision sets $\cD_t$ at each round are not identical, which is more challenging than stochastic linear bandit setting in \citet{lee2021achieving}, where the decision set is pregiven before the execution of the algorithm and fixed during the execution of the algorithm. 
Therefore, our result and that in \citet{lee2021achieving} are not directly comparable. 
\end{remark}

\begin{remark}
It is worth noting that our regret upper bound also holds in a stronger model than the one described in Section \ref{section:3}, where the adversary can even decide the decision set $\cD_t$ at each round $t$ since our regret bound holds without any assumption on the decision sets.  
\end{remark}

\begin{remark}
Similar to Remark \ref{remark:4.3}, if we trivially upper bound $\sigma_t$'s by $R$, the regret of Algorithm \ref{alg:1} degenerate to 
$ \tilde{O}\left(C^2dR\sqrt{T}\right)$.
\end{remark}


We also provide a gap-dependent regret bound. 
\begin{theorem}\label{thm:gap dependent regret}
        Suppose that $C = \Omega(1)$, $R = \Omega(1)$, for all $t \ge 1$ and all $\ab \in \cD_t$, $\langle \ab, \bmu^* \rangle \in [-1, 1]$. Then with probability at least $1 - 3 \delta$, the regret of Algorithm \ref{alg:1} is bounded as follows: $$\regret(T) =  \tilde{O}\left(\frac{1}{\Delta}\cdot C^2R^2d + \frac{1}{\Delta}\cdot d^2C^2 \max_{t \in[T]} \sigma_t^2\right). $$ 
\end{theorem}

\begin{remark}
Theorem~\ref{thm:gap dependent regret} automatically suggests an $\tilde O(R^2d^2C^2/\Delta)$ regret bound, by the fact $\sigma_t = O(R)$. Compared with previous result $\tilde O(d^{5/2}C/\Delta + d^6/\Delta^2)$ \citep{lee2021achieving}, our result has a better dependence on the dimension $d$ but a worse dependence on the corruption level $C$. As Remark \ref{remark:regret1} suggests, we focus on a more challenging linear contextual bandit setting, and the worse dependence on $C$ might be due to this.  
\end{remark}

\section{Proof Outline} \label{section:6}


First, we have the following lemma, which is a corruption-tolerant variant of Bernstein inequality for self-normalized vector-valued martingales introduced in \cite{Zhou2020NearlyMO}. 

\begin{lemma}[Bernstein inequality for vector-valued martingales with corruptions]\label{sec6:Bernstein with corruptions(bandit)}
    Let $\{\mcal{G}_t\}_{t = 1}^\infty$ be a filtration, $\{\xb_t, \eta_t\}_{t \ge 1}$ a stochastic process so that $\mathbf{x}_t \in \mbb{R}^d$ is $\mcal{G}_t$-measurable and $\eta_t \in \mbb{R}$ is $\mcal{G}_{t + 1}$-measurable. Fix $R, L, \sigma, \lambda > 0, \ \bmu^* \in \mbb{R}^d$. For $t \ge 1$ let ${y}_t^\stoch = \langle \mathbf{\bmu}^*, {\mathbf{x}}_t \rangle + {\eta}_t$ and suppose that ${\eta}_t, {\mathbf{x}}_t$ also satisfy $$|{\eta}_t| \le R, \mbb{E}[{\eta}_t|\mcal{G}_t] = 0, \mbb{E}[{\eta}_t^2|\mcal{G}_t] \le \sigma^2, \|{\mathbf{x}}_t\|_2 \le L. $$ 

    Suppose $\{y_t\}$ is a sequence such that $\sum_{i = 1}^t |y_i - y_i^\stoch| = C(t)$ for all $t \ge 1$.  Then, for any $0 < \delta < 1$, with probability at least $1 - \delta$ we have $\forall t > 0$, \begin{align*} \|\bmu_t - \bmu^*\|_{\mathbf{Z}_t} \le \beta_t + C(t) + \sqrt{\lambda} \|\bmu^*\|_2, \end{align*} where for $t \ge 1$, $\bmu_t = \mathbf{Z}_t^{-1}\mathbf{b}_t$, $\mathbf{Z}_t = \lambda \mathbf{I} + \sum_{i = 1}^t {\mathbf{x}}_i {\mathbf{x}}_i^\top$, $\mathbf{b}_t = \sum_{i = 1}^t {y}_i{\mathbf{x}}_i$, and $$\beta_t = 8\sigma\sqrt{d\log \frac{d\lambda + tL^2}{d\lambda}\log(4t^2 / \delta)} + 4R\log(4t^2 / \delta). $$
\end{lemma}



Next, we have that with high probability, all the levels satisfying $\ell \ge \ell^*$ are only influenced by a limited amount of corruptions as mentioned in Section \ref{section:4}. 

\begin{lemma}\label{sec6:sb-corrupt-bandit}
    Let $\text{Corruption}_{t, \ell}$ be defined in \eqref{def:sbcorruption}. Then we have with probability at least $1 - \delta$, for all $\ell \ge \ell^*$, $t \ge 1$:
    $$\text{Corruption}_{t, \ell} \le \oline{C}_\ell = \log(2\ell^2 / \delta) + 3. $$
    
    Let $\cE_{\text{sub}}$ be the event that the above inequality holds. 
\end{lemma}


We define the following event to further show that our candidate confidence sets with $\ell \ge \ell^*$ are ``robust'' enough, i.e., $\cC_{t, \ell}$ contains $\bmu^*$ with high probability. 

\begin{definition} 
Let $\ell^*$ be defined in \eqref{robustlayer}. 
We introduce the event $\cE_1$ as follows. 


\begin{align} \cE_{1} := \{\forall \ell \ge \ell^*\ \text{and}\ t\ge 1, \|\bmu^* - \bmu_t\|_{\bSigma_t} \le \beta_{t, \ell} \ \text{and}\ \|\bmu^* - \bmu_{t,\ell}\|_{\bSigma_{t, \ell}} \le \gamma_{t, \ell} \}. \label{sec6:event:glsb}
\end{align}
where $\beta_{t, \ell}$, $\gamma_{t, \ell}$ are defined in \eqref{def:beta} and \eqref{def:gamma}.

\end{definition}

Next lemma suggests that the event $\cE_1$ happens with high probability.

\begin{lemma} \label{sec6:confidence}
Let $\cE_{1}$ be defined in \eqref{sec6:event:glsb}. For any $0 < \delta < 1 / 3$, we have $\PP(\cE_{1}) \ge 1 - 3\delta$. 
\end{lemma}

For simplicity, we define $\ab_{t, \ell} = \argmax_{\ab \in \cD_t}\max_{\bmu \in \cC_{t, \ell}}\langle \bmu, \ab \rangle$ for each level $\ell$. $\ab_t$ can be seen as an action vector randomly chosen from $\ab_{t, \ell}$, $\ell \in [\ell_{\max}]$. Next two lemmas suggest that under event $\cE_1$, at each round, the gap between the optimal reward and the selected reward can be upper bounded by some bonus terms related to $\ab_{t, \ell}$. 

\begin{lemma} \label{sec6:error:ell}
    On event $\cE_1$, if $f(t) \le \ell^*$, we have $ \langle\ab_t^* -  \ab_t, \bmu^* \rangle \le 2\beta_{t, \ell^*} \|\ab_t\|_{\bSigma_{t}^{-1}} + 2\beta_{t, \ell^*} \|\ab_{t, \ell^*}\|_{\bSigma_t^{-1}}. $ 
\end{lemma}

\begin{lemma}\label{sec6:error:ell2} On event $\cE_1$, if $f(t) = \ell > \ell^*$, we have $\langle \ab_t^* - \ab_t, \bmu^* \rangle \le 2 \gamma_{t, \ell} \|\ab_{t}\|_{\bSigma_{t, \ell}^{-1}}. $
\end{lemma}

Equipped with the above lemmas, we provide the proof sketch of Theorem \ref{thm:regret}. 
\begin{proof}[Proof sketch of Theorem \ref{thm:regret}
]
    Suppose $\cE_1$ occurs. The main idea to bound the regret is to decompose the total rounds $[T]$ into two non-overlapping parts, based on which individual learner is selected at that round. In detail, we have
    \begin{align}
        \regret(T)\notag &= \EE\left[\sum_{t = 1}^T \left(\langle \ab_t^*, \bmu^*\rangle - \langle \ab_t, \bmu^* \rangle \right)\right] \notag
        \\&= \underbrace{\EE\left[\sum_{t = 1}^T \mathds{1}(f(t) \le \ell^*)\left(\langle \ab_t^*, \bmu^*\rangle - \langle \ab_t, \bmu^* \rangle \right)\right]}_{I_1}\notag \\\qquad &+ \sum_{\ell = \ell^* + 1}^{\ell_{\max}}\underbrace{\EE\left[\sum_{t = 1}^T \mathds{1}(f(t)= \ell)\left(\langle \ab_t^*, \bmu^*\rangle - \langle \ab_t, \bmu^* \rangle \right)\right]}_{I_2(\ell)}.  \label{sketch:I1+I2}
    \end{align} 
    Here $I_1$ represents the regret where the the ``low-level'' learner is selected, and the corruption level is beyond the learner level. In this case, by Lemma \ref{sec6:error:ell}, we can directly show that
    \begin{align}
        I_1 \le \EE\Big[\sum_{t = 1}^T \mathds{1}(f(t) \le \ell^*) \min\left\{2, 2\beta_{t, \ell^*}\|\ab_{t, \ell^*}\|_{\bSigma_t^{-1}} + 2\beta_{t, \ell^*}\|\ab_{t}\|_{\bSigma_t^{-1}} \right\}\Big].\label{xx1}
    \end{align}
    We further bound \eqref{xx1}. Let $\cF_t$ be the $\sigma$-algebra generated by $\ab_s, r_{s}, \sigma_s, f(s)$ for $s \le t - 1$. Then by the property of our partition scheme (note that $\PP(f(t) = \ell^*) = 2^{-\ell^*}$), we can show that 
    $\EE\left[\mathds{1}(f(t)\le \ell^*) \|\ab_{t, \ell^*}\|_{\bSigma_t^{-1}} | \cF_{t}\right] \le 2^{\ell^*} \EE\left[\|\ab_{t}\|_{\bSigma_t^{-1}} | \cF_{t}\right]
    $. Therefore, we can further bound $I_1$ by 
    
    
    \begin{align}
        I_1 \le 4\cdot 2^{\ell^*}\EE\underbrace{\left[\sum_{t = 1}^T \min\left\{2, \beta_{T, \ell^*} \|\ab_t\|_{\bSigma_t^{-1}}\right\}\right]}_{I_3}. \label{sketch:I1:4.12}
    \end{align}
    To further bound $I_3$, we split $[T]$ into 2 parts, $\cI_1 = \{t \in [T] | \|\ab_t / \oline{\sigma}_t\|_{\bSigma_t^{-1}} > 1\}, \cI_2 = \{t \in [T] | \|\ab_t / \oline{\sigma}_t\|_{\bSigma_t^{-1}} \le 1\}$. To bound $\cI_1$ part, the intuition is that the cardinality of $\cI_1$ is bounded, and the sum of terms with $t \in \cI_2$ can be bounded using Cauchy-Schwarz inequality. 
    

    \begin{align}
        &\sum_{t \in \cI_1} \min\left\{2, \beta_{T, \ell^*}\|\ab_t\|_{\bSigma_t^{-1}}\right\} \le 2 |\cI_1| \notag 
        \\&\le 2\sum_{t = 1}^T \min\left\{1, \|\ab_t / \oline{\sigma}_t\|_{\bSigma_t^{-1}}^2\right\} \notag
        \\&\le 4d \log \frac{(R + 1)^2\lambda + TA^2}{(R + 1)^2 \lambda},  \label{sketch:I3:4.15}
    \end{align}
    where the first inequality holds since $\min\left\{2, \beta_{T, \ell^*}\|\ab_t\|_{\bSigma_t^{-1}}\right\} \le 2$, the second inequality follows from the definition of $\cI_1$, and the third inequality holds by Lemma \ref{lemma:summation11}. 
To bound $\cI_2$ part, we have


    \begin{align}
        &\sum_{t \in \cI_2} \min\left\{2, \beta_{T, \ell^*}\|\ab_t\|_{\bSigma_t^{-1}}\right\} \notag \\&\le \beta_{T, \ell^*} \sqrt{\sum_{t \in \cI_2} \oline{\sigma}_t^2} \cdot \sqrt{\sum_{t \in \cI_2} \min\left\{1, \|\ab_t / \oline{\sigma}_t\|_{\bSigma_t^{-1}}^2\right\}} \notag
        \\&\le \beta_{T, \ell^*} \sqrt{(R + 1)^2T / d + \sum_{t = 1}^T {\sigma}_t^2}\cdot \tilde{O}(\sqrt{d}),  \label{sketch:I3:4.17}
    \end{align}
    where the first inequality follows from Cauchy-Schwarz inequality, the second inequality follows from the definition of $\oline{\sigma}_t$ and Lemma \ref{lemma:summation11}. 

    Substituting \eqref{sketch:I3:4.15} and \eqref{sketch:I3:4.17} into \eqref{xx1},  we have 
\begin{align}
    I_1 = \tilde{O}\left(C^2d\sqrt{\sum_{t = 1}^T \sigma_t^2} + C^2R\sqrt{dT}\right).\label{yy0}
\end{align}
    Now it remains to bound $I_2(\ell)$. By Lemma \ref{sec6:error:ell2}, we have 
    \begin{align}
       I_2(\ell) &\le 2 \EE\underbrace{\left[\sum_{t = 1}^T \mathds{1}(f(t) = \ell) \min\left\{1,  \gamma_{t, \ell} \|\ab_{t, \ell}\|_{\bSigma_{t, \ell}^{-1}}\right\}\right]}_{I_4} \notag \\&= \tilde{O}\left(R\sqrt{Td} + d\sqrt{\sum_{t = 1}^T \sigma_t^2}\right),\label{yy1}
    \end{align}
    where the second equality can be proved by an analysis similar to that of \eqref{sketch:I3:4.15} and \eqref{sketch:I3:4.17}. Finally, substituting \eqref{yy0} and \eqref{yy1} into \eqref{sketch:I1+I2} completes our proof. 

    
\end{proof}

\section{Model Selection-based Approaches for Unknown Corruption}

Over the past years, there has been works which design adaptive master algorithms that perform nearly as well as the best base algorithm \citep{odalric2011adaptive, agarwal2017corralling, cheung2019learning, locatelli2018adaptivity, foster2019model, pacchiano2020model, foster2021adapting}. This is also known as model selection in bandits. With these approaches, one can design an algorithm that can adapt to the ground-truth model even when it is unknown. In this section, we discuss what regret upper bound can be achieved if we apply these model selection-based approaches to handle the unknown-$C$ case in our problem. 

\subsection{Base Algorithms}

We have shown in Section \ref{section:warmup} that if $C$ is known a simple algorithm can achieve a regret upper bound of $\tilde{O}(CRd\sqrt{T})$. Hence, it is natural to apply a master algorithm to do model selection from $\lceil \log_2T \rceil$ base algorithms, where the $i$-th base algorithm assumes that the corruption level is under $2^i$. 
\subsection{Choosing Master Algorithm}

\cite{odalric2011adaptive} are probably the earliest work to study model-selection in the bandit problem. They essentially used EXP4 as the master algorithm and their master algorithm suffers $\tilde{O}(T^{2/3})$ regret. \cite{agarwal2017corralling} improved and generalized their result with a master algorithm called CORRAL. 
\citet{agarwal2017corralling}, showed that CORRAL can achieve a regret upper bound of $\tilde{O}(\sqrt{MT} + MR_i(T))$, where $R_i(T)$ is the regret bound of the $i$-th algorithm, $M$ is the number of the base algorithms. To achieve this bound, however, we need to set the learning rate $\eta$ (a hyper-parameter in CORRAL) to be $\min\left\{\frac{1}{40 R_i(T) \ln T}, \sqrt{\frac{M}{T}}\right\}$. This indicates that we cannot apply the original CORRAL algorithm here since $R_i(T)$ has an multiplicative dependence on the unknown $C$ if $2^i \ge C$ (if $2^i < C$, $R_i(T)$ has no theoretical gaurantees). 

Fortunately, \cite{pacchiano2020model} resolved this issue when $R_i(T)$ is unknown a priori. They proposed a stochastic version of CORRAL with a regret guarantee under certain assumptions. 

\begin{assumption}[Section 2, \cite{pacchiano2020model}] \label{iidcontext}
    Let $\cA \subseteq \RR^d$ be a set of actions. Let $S$ be the set of all subsets of $\cA$ and let $\cD_S$ be a distribution over $S$. Assume that the decision set at each round is sampled independently from $\cD_S$. 
\end{assumption}

\begin{theorem}[Theorem 3.2, \cite{pacchiano2020model}] \label{thm:corral}
    If the regret of the optimal base is upper bounded by $U_*(T, \delta) = O(c(\delta)T^\alpha)$ for some function $c: \RR \to \RR$ and constant $\alpha \in [\frac{1}{2}, 1)$, the regrets of master algorithms EXP3.P and CORRAL are: 
    
    \vspace{0.5em}
    \centering
    \begin{tabular}{|c|c|c|}
        \hline
        & Known $\alpha$ and $c(\delta)$ & Known $\alpha$,  Unknown $c(\delta)$ \\
        \hline
        EXP3.P & $\tilde{O}(T^{\frac{1}{2 - \alpha}}c(\delta)^\frac{1}{2 - \alpha})$ & $\tilde{O}(T^{\frac{1}{2 - \alpha}} c(\delta))$ \\
        \hline
        CORRAL & $\tilde{O}(T^\alpha c(\delta))$ & $\tilde{O}(T^{\alpha} c(\delta)^\frac{1}{\alpha})$ \\ \hline
    \end{tabular}
\end{theorem}

To achieve a $\sqrt{T}$ regret upper bound, we can use the stochastic CORRAL \citep{pacchiano2020model} as the master algorithm. By Theorem \ref{thm:corral}, the resulting regret bound will be $\tilde{O}(C^2d^2 R \sqrt{T})$. This is actually worse than the regret of our Algorithm \ref{alg:1} by a factor of $d$. More importantly, this regret only holds under Assumption \ref{iidcontext}, which is a very restrictive assumption that significantly downgrades the generality of linear contextual bandits. 

\section{Experiments}

In this section, we conduct experiments and evaluate the performance our algorithms Multi-level weighted OFUL and Robust weighted OFUL, along with the baselines, OFUL \citep{AbbasiYadkori2011ImprovedAF}, Weighted OFUL \citep{Zhou2020NearlyMO} and the greedy algorithm proposed by \cite{pmlr-v130-bogunovic21a} under different corruption levels. We repeat each baseline algorithm for 10 times and plot their regrets w.r.t. number of rounds in Figure \ref{fig:1}.

\subsection{Experimental Setup}

Following \cite{pmlr-v130-bogunovic21a}, we let the adversary always corrupt the first $k$ rounds, and leave the rest $T - k$ rounds intact. According to our definition in \eqref{def:corruption}, our design can simulate the cases where corruption level is $2k$. 

\noindent\textbf{Model parameters. } Recall that corrupted linear contextual bandits defined in Section \ref{section:3}, we consider $B = 1$, $A = 1$ 
$d = 20$ and $R = 0.5$ and fix $\bmu^*$ as $\left(\frac{1}{\sqrt{d}}, \cdots, \frac{1}{\sqrt{d}}\right)^\top$. We set $\sigma_t$ as a random variable which is independently and uniformly chosen from $[0, 0.05]$ in each round $t$. Note that $\langle \ab, \bmu^*\rangle \in [-1, 1]$ always hold for any eligible $\ab$ under our setting of parameters. 

\noindent\textbf{Attack method. } In the first $k$ rounds, the adversary always trick the learner by flipping the value of $\bmu^*$, i.e., $r_t(\ab) = -\langle \ab, \bmu^* \rangle + \epsilon_t(\ab)$ for all $t \in [k]$ and $\ab \in \cD_t$. 

\noindent\textbf{Decision set. } We consider $|\cD_t| = 20$ for all $t \ge 1$. In each of the first $k$ rounds, we generate the 20 actions in $\cD_t$ independently, each having entries drawn i.i.d. from the uniform distribution on $\left[-\frac{1}{\sqrt{d}}, \frac{1}{\sqrt{d}}\right]$. For the following uncorrupted rounds, however, we use a fixed $\cD$ generated in the same way. 

Intuitively, non-robust algorithm will ``learn'' the flipped $\bmu^*$ faster with diversified action vectors. As a result, the learner is likely to select the same nonoptimal action for a huge number of rounds afterwards, making it even more difficult to learn the true $\bmu^*$.  

\noindent\textbf{Noise synthesis. } We generate identical noises $\epsilon_t$ for all $\ab \in \cD_t$ at each round $t$, i.e., $\epsilon_t(\ab) = \epsilon_t$. To generate $\epsilon_t$, we first generate $\epsilon_t'$ subject to $\cN(0, \sigma_t^2)$ and let \[\epsilon_t = \begin{cases}
-R, & \epsilon_t' < -R \\
R, & \epsilon_t' > R \\
\epsilon_t', & \text{otherwise}
\end{cases}.\] 

\subsection{Results and Discussion}

\begin{figure}[ht!]
\begin{center}
\subfigure[$C = 0$]{\includegraphics[width=0.32\textwidth]{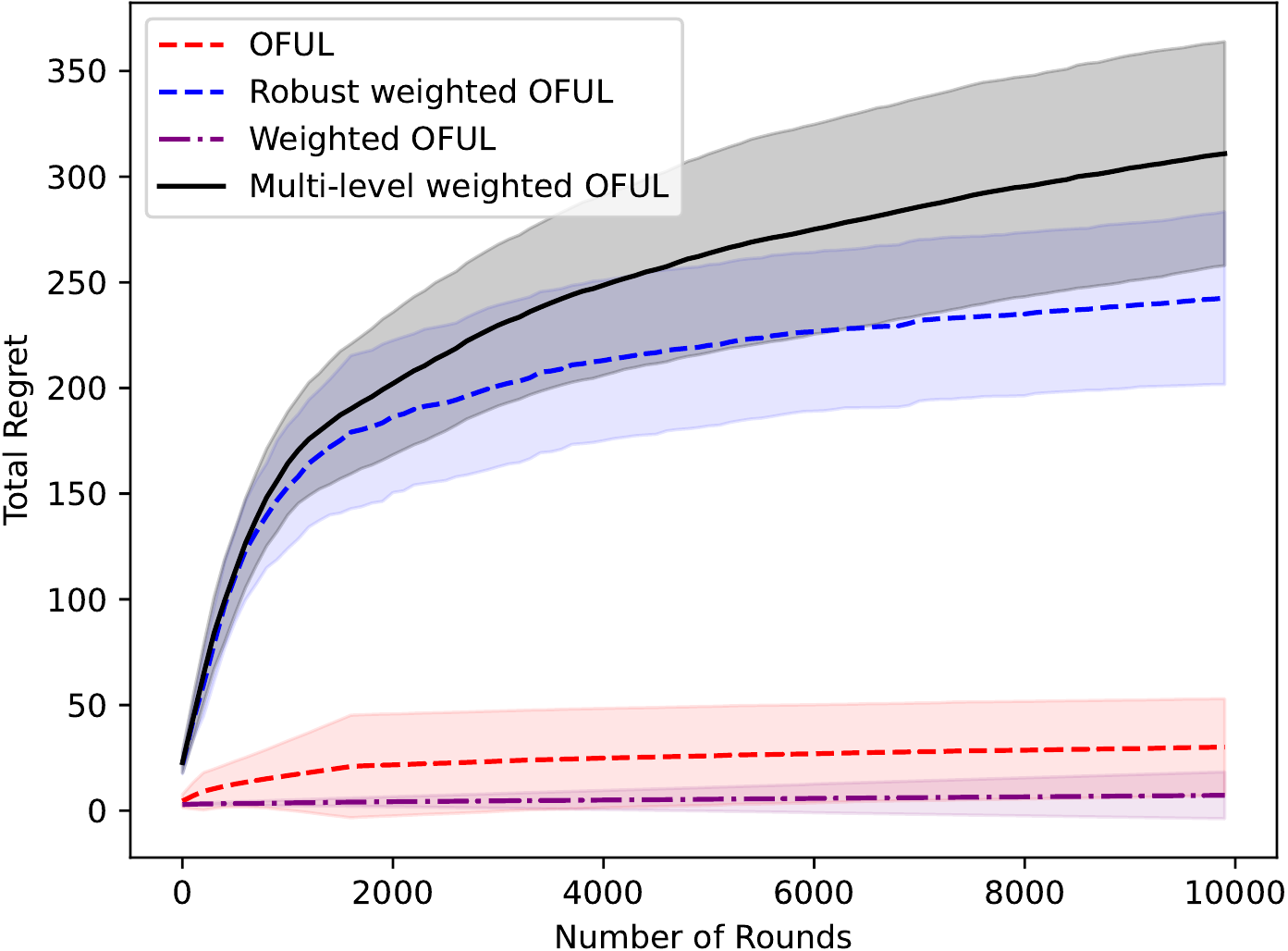}
		\label{fig:a}}
\subfigure[$C = 60$]{\includegraphics[width=0.32\textwidth]{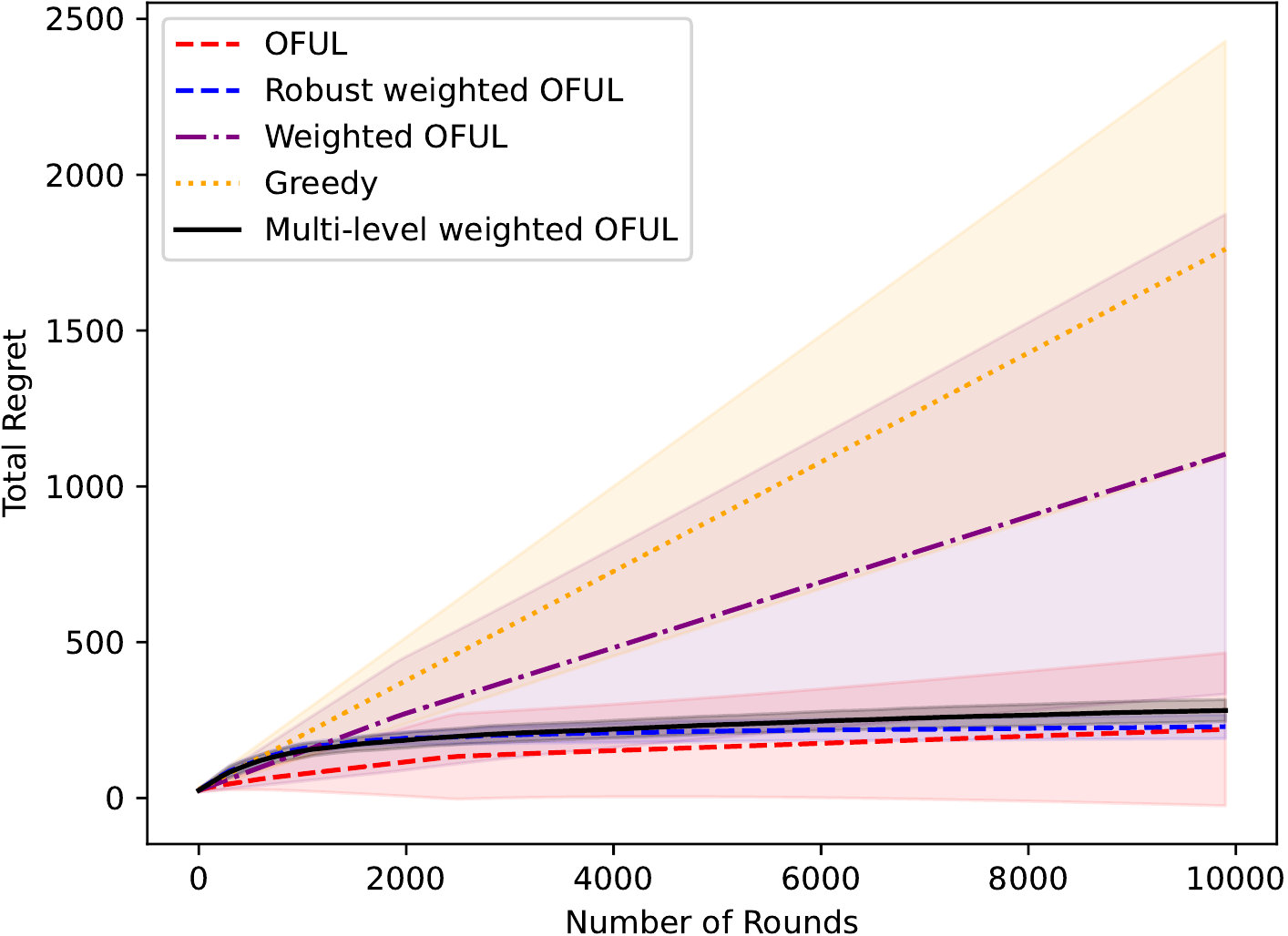}
		\label{fig:b}}
\subfigure[$C = 180$]{\includegraphics[width=0.32\textwidth]{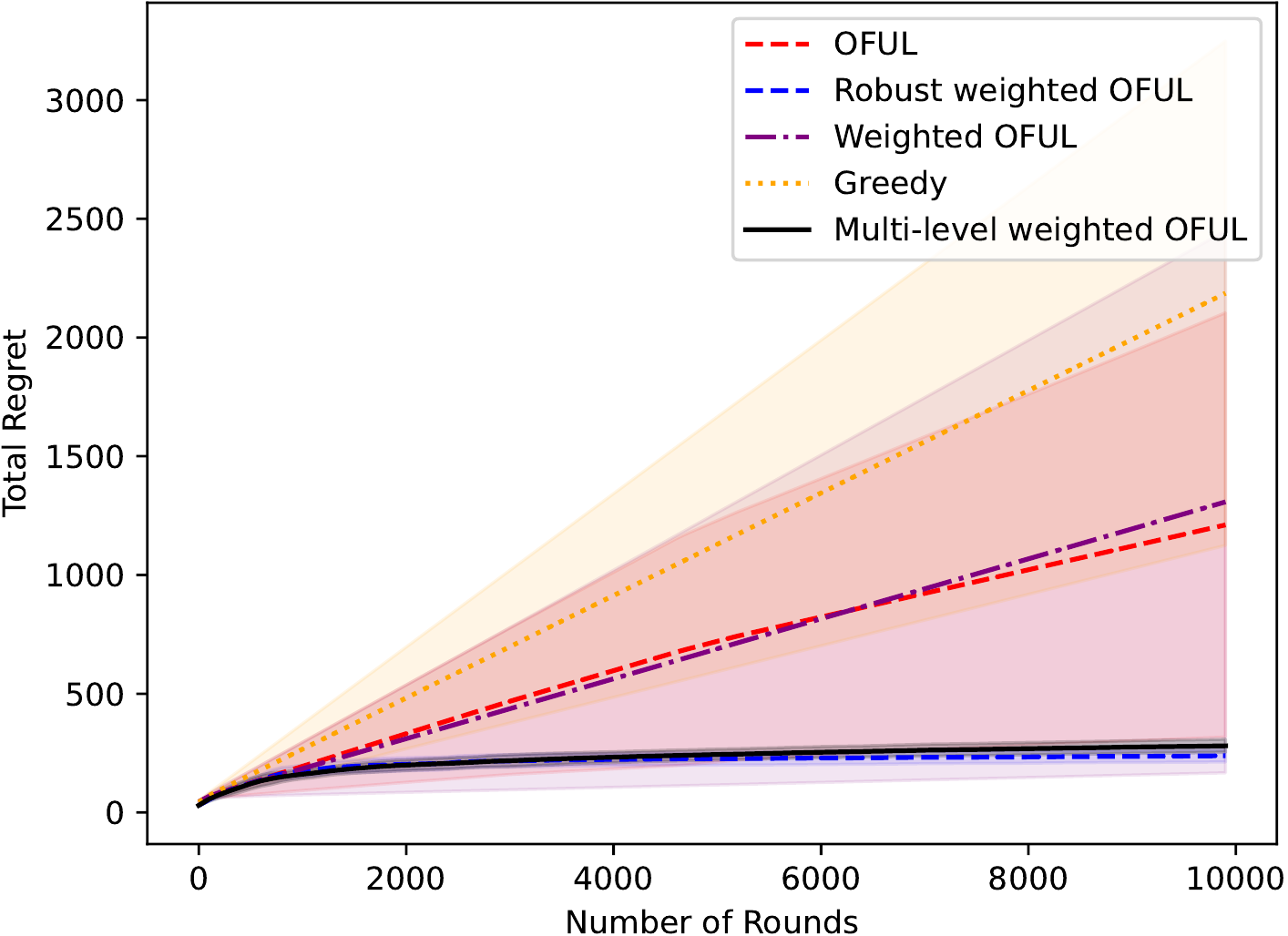}
		\label{fig:c}}
\subfigure[$C = 300$]{\includegraphics[width=0.32\textwidth]{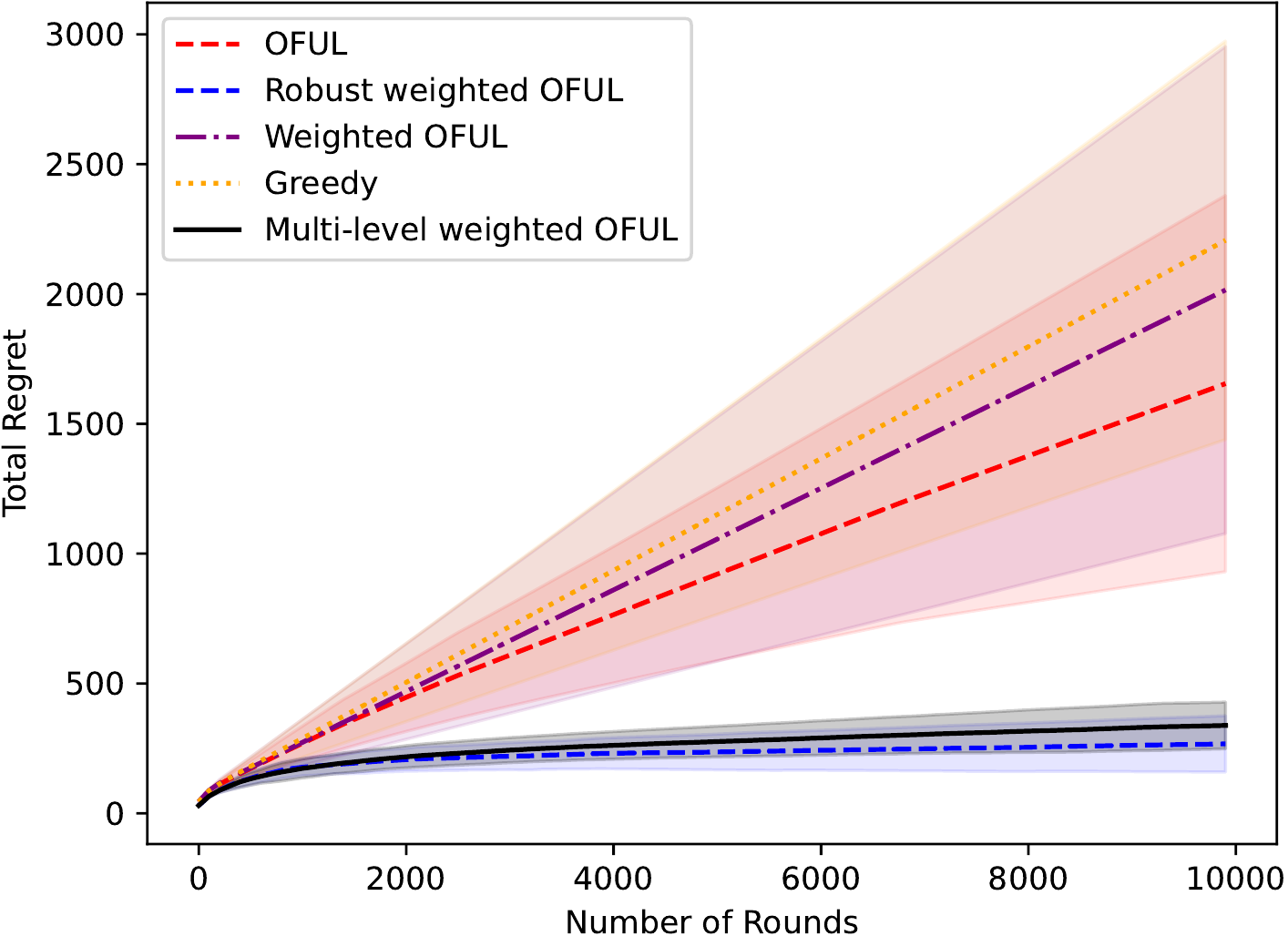}
		\label{fig:d}}
\subfigure[$C = 600$]{\includegraphics[width=0.32\textwidth]{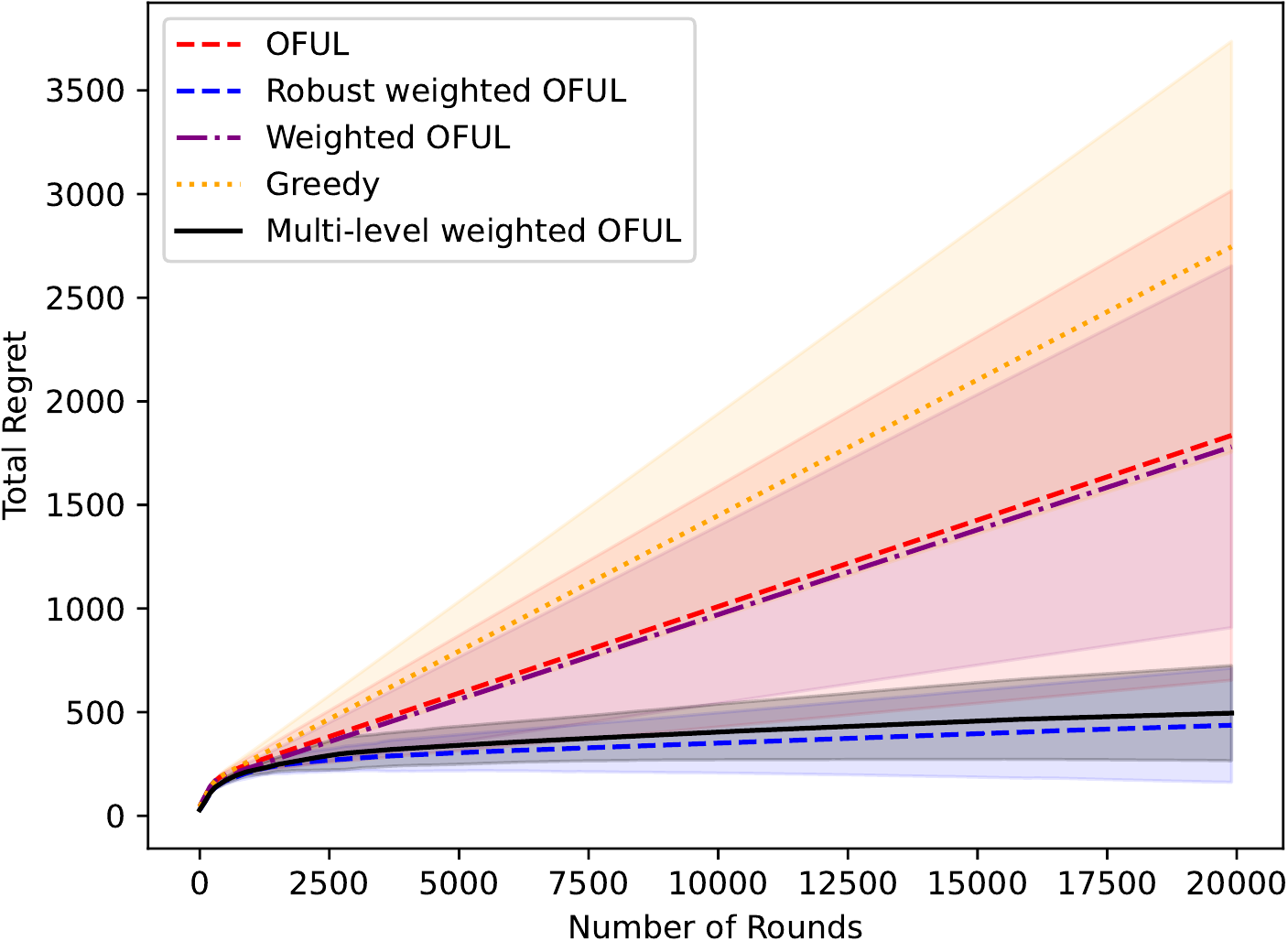}
		\label{fig:e}}
\subfigure[$C = 900$]{\includegraphics[width=0.32\textwidth]{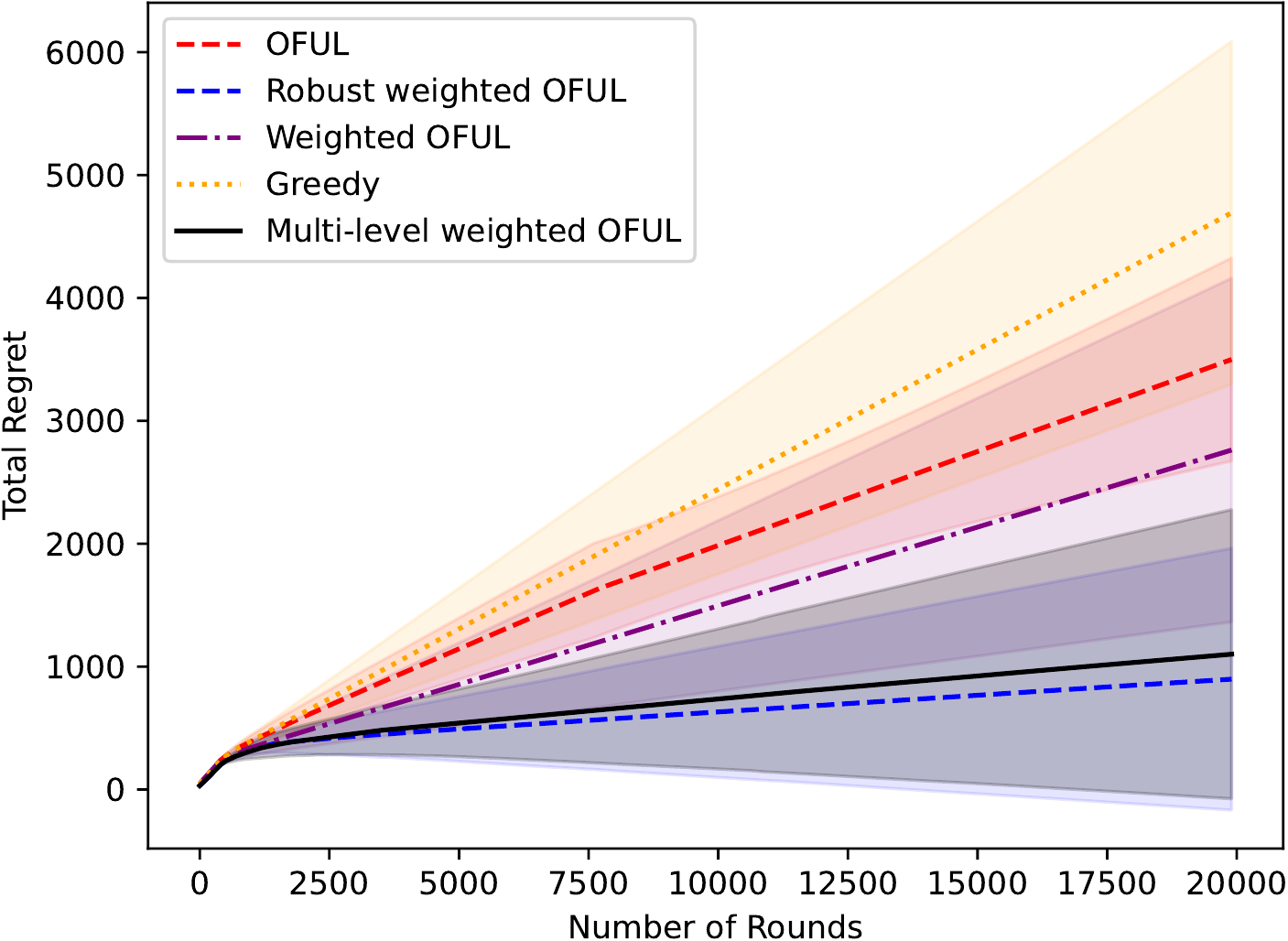}
		\label{fig:f}}
\caption{Regret plot against number of rounds under corruption level from 0 to 900 averaged in 10 trials. \label{fig:1}}
\end{center}
\end{figure}


We plot the regret with respect to the number of rounds in Figure \ref{fig:1}. The results are averaged over 10 trials. In the setting where $C = 0$ (Figure \ref{fig:a}), we only plot the regret of OFUL, Weighted OFUL and Multi-level weighted OFUL, robust weighted OFUL and do not plot the regret of the greedy algorithm since its regret is much worse than the other four algorithms. 

We have the following observations from Figure \ref{fig:1}. For the corruption-free case $C = 0$ (Figure~\ref{fig:a}), our proposed Multi-level weighted OFUL behaves worse than Weighted OFUL and OFUL, which is not surprising since Multi-level weighted OFUL has additional algorithm design to deal with the corruption and it may pay additional price in regret in the absence of corruption. Weighted OFUL outperforms OFUL remarkably since it takes advantage of the information concerning the variance of noise. For the corruption case with $C \ge 180$ (Figure \ref{fig:c} to \ref{fig:f}), our Multi-level weighted OFUL outperforms other baseline algorithms by a large margin and suffers from minor additional regret compared with Robust weighted OFUL, which suggests that it can deal with the corruption successfully. 

\section{Conclusion and Future Work}

In this paper, we have considered the linear contextual bandit problem in the presence of adversarial corruptions. We propose a Multi-level weighted OFUL algorithm, which is provably robust to the adversarial attacks. 

We leave it as an open question that whether the multiplicative dependence on $C^2$ in the regret upper bounds can be removed without making additional assumptions in our setting. 


\appendix

\section{Proofs from Section \ref{section:warmup}} \label{proof:warmup}

\subsection{Proof of Lemma \ref{lemma:oful1}}

This lemma can be proved by a direct application of Lemma \ref{Bernstein with corruptions(bandit)}. 

\begin{proof}
    Applying Lemma \ref{Bernstein with corruptions(bandit)}, we have \begin{align*}
        \|\bmu_t - \bmu^*\|_{\mathbf{\Sigma}_t} &\le 8\sqrt{d\log \frac{(R + 1)^2\lambda + tA^2}{(R + 1)^2\lambda}\log(4t^2 / \delta)} + 4\sqrt{d}\log(4t^2 / \delta) + C{\sqrt{d}} + \sqrt{\lambda} \|\bmu^*\|_2 \\&\le 
        8\sqrt{d\log \frac{(R + 1)^2\lambda + tA^2}{(R + 1)^2\lambda}\log(4t^2 / \delta)} + 4\sqrt{d}\log(4t^2 / \delta) + C{\sqrt{d}} + \sqrt{\lambda} B
    \end{align*} for all $t \ge 1$
    with probability at least $1 - \delta$. 
\end{proof}

\subsection{Proof of Theoreom \ref{thm:alg:1}}

\begin{proof}

Suppose $\bmu^* \in \cC_t$ for all $t \ge 1$. Then we can bound the total regret as follows: 
\begin{align} 
    \regret(T) &= \EE\left[\sum_{t = 1}^T \left(\langle \ab_t^*, \bmu^* \rangle - \langle \ab_t, \bmu^* \rangle \right)\right] \notag
    \\&\le \EE\left[\sum_{t = 1}^T \min\left(2, \langle \ab_t, \bmu_t \rangle - \langle \ab_t, \bmu^* \rangle\right)\right] \notag
    \\&\le \EE\left[\sum_{t = 1}^T \min\left(2, \|\ab_t\|_{\Sigma_t^{-1}} \alpha_t\right)\right] \notag
    \\&\le  \EE\underbrace{\left[\sum_{t = 1}^T \min \left(2, \alpha_T\|\ab_t\|_{\Sigma_t^{-1}}\right)\right]}_{I} \label{eq:B.1}
\end{align}
where the first inequality holds since $\bmu^* \in \cC_t$ and $\ab_t = \argmax_{\ab \in \cD_t} \max_{\bmu \in \cC_t} \langle \ab, \bmu \rangle$, the second inequality holds due to the definition of $\cC_t$ in \eqref{eq:def:ct} and the third inequality follows from the monotonicity of $\{\alpha_t\}$. 

We split $[T]$ into two parts to bound $I$. 

Let $\cI_1 = \{t \in [T] | \|\ab_t / \oline{\sigma}_t\|_{\bSigma_t^{-1}} > 1\}, \cI_2 = \{t \in [T] | \|\ab_t / \oline{\sigma}_t\|_{\bSigma_t^{-1}} \le 1\}. $

We calculate \begin{align} 
|\cI_1| &\le \sum_{t = 1}^T \min\left(1, \|\ab_t / \oline{\sigma}_t\|_{\bSigma_{t}^{-1}}^2\right) \notag
\\&\le 2d \log \frac{(R + 1)^2\lambda + TA^2}{(R + 1)^2 \lambda} \label{eq:reg:part1}
\end{align} where the second inequality holds due to Lemma \ref{lemma:summation11}. 

\begin{align} 
    \sum_{t \in \cI_2} \min\left(2, \alpha_T \|\ab_t\|_{\bSigma_t^{-1}}\right) &\le \alpha_T \sum_{t \in \cI_2} \|\ab_t\|_{\bSigma_t^{-1}} \notag
    \\&= \alpha_T \sum_{t \in \cI_2} \oline{\sigma}_t \|\ab_t / \oline{\sigma}_t\|_{\bSigma_t^{-1}} \notag
    \\&\le \alpha_T \sqrt{\sum_{t \in \cI_2} \oline{\sigma}_t^2} \cdot \sqrt{\sum_{t \in \cI_2} \|\ab_t / \oline{\sigma}_t\|_{\bSigma_t^{-1}}^2} \notag
    \\&\le \alpha_T \left(\sqrt{\sum_{t = 1}^T \sigma_t^2} + \sqrt{R^2 T / d}\right) \cdot \sqrt{2d \log \frac{(R + 1)^2\lambda + TA^2}{(R + 1)^2 \lambda}} \label{eq:reg:part2}
\end{align} where the second inequality follows from Cauchy Schwarz inequality, the third inequality holds due to the definition of $\oline{\sigma}_t$ and Lemma \ref{lemma:summation11}. 

We can further complete the proof by substituting \eqref{eq:reg:part1} and \eqref{eq:reg:part2} into \eqref{eq:B.1}. 

\end{proof}

\section{Proofs of Results in Section \ref{section:4}} \label{section:7}

\subsection{Proof of Theorem \ref{thm:regret}}

We first prove the following lemma which is a corruption-tolerant variant of Bernstein inequality for self-normalized vector-valued martingales introduced in \cite{Zhou2020NearlyMO}. 

\begin{lemma}[Restatement of Lemma \ref{sec6:Bernstein with corruptions(bandit)}]\label{Bernstein with corruptions(bandit)}
    Let $\{\mcal{G}_t\}_{t = 1}^\infty$ be a filtration, $\{\xb_t, \eta_t\}_{t \ge 1}$ a stochastic process so that $\mathbf{x}_t \in \mbb{R}^d$ is $\mcal{G}_t$-measurable and $\eta_t \in \mbb{R}$ is $\mcal{G}_{t + 1}$-measurable. Fix $R, L, \sigma, \lambda > 0, \ \bmu^* \in \mbb{R}^d$. For $t \ge 1$ let ${y}_t^\stoch = \langle \mathbf{\bmu}^*, {\mathbf{x}}_t \rangle + {\eta}_t$ and suppose that ${\eta}_t, {\mathbf{x}}_t$ also satisfy $$|{\eta}_t| \le R, \mbb{E}[{\eta}_t|\mcal{G}_t] = 0, \mbb{E}[{\eta}_t^2|\mcal{G}_t] \le \sigma^2, \|{\mathbf{x}}_t\|_2 \le L. $$ 

    Suppose $\{y_t\}$ is a sequence such that $\sum_{i = 1}^t |y_i - y_i^\stoch| = C(t)$ for all $t \ge 1$.  Then, for any $0 < \delta < 1$, with probability at least $1 - \delta$ we have $\forall t > 0$, \begin{align*} \|\bmu_t - \bmu^*\|_{\mathbf{Z}_t} \le \beta_t + C(t) + \sqrt{\lambda} \|\bmu^*\|_2, \end{align*} where for $t \ge 1$, $\bmu_t = \mathbf{Z}_t^{-1}\mathbf{b}_t$, $\mathbf{Z}_t = \lambda \mathbf{I} + \sum_{i = 1}^t {\mathbf{x}}_i {\mathbf{x}}_i^\top$, $\mathbf{b}_t = \sum_{i = 1}^t {y}_i{\mathbf{x}}_i$, and $$\beta_t = 8\sigma\sqrt{d\log \frac{d\lambda + tL^2}{d\lambda}\log(4t^2 / \delta)} + 4R\log(4t^2 / \delta). $$
\end{lemma}
\begin{proof}
    See Appendix \ref{proof:6.1}. 
\end{proof}



Then we prove that with high probability, all the level $\ell \ge \ell^*$ only influenced by limited amount of corruptions as mentioned in Section \ref{section:4}. 

\begin{lemma}[Restatement of Lemma \ref{sec6:sb-corrupt-bandit}] \label{sb-corrupt-bandit}
    Let $\text{Corruption}_{t, \ell}$ be defined in \eqref{def:sbcorruption}. Then we have with probability at least $1 - \delta$, for all $\ell \ge \ell^*$, $t \ge 1$:
    $$\text{Corruption}_{t, \ell} \le \oline{C}_\ell = \log(2\ell^2 / \delta) + 3. $$
    
    We denote by $\cE_{\text{sub}}$ the event that the above inequality holds. 
\end{lemma}

\begin{proof}
    The proof of this lemma is based on Lemma \ref{sb-corrupt-bandit:helper} introduced by \cite{lykouris2018stochastic}; for details see Appendix \ref{proof:6.2}. 
\end{proof}

We define the following event to further show that our candidate confidence sets with $\ell \ge \ell^*$ are ``robust'' enough, i.e. $\cC_{t, \ell}$ contains $\bmu^*$ with high probability. 

\begin{definition} 
Let $\ell^*$ be defined in \eqref{robustlayer}. 
We introduce the event $\cE_1$ as follows. 

\begin{align} \cE_{1} := \left\{\forall \ell \ge \ell^*\ \text{and}\ t\ge 1, \|\bmu^* - \bmu_t\|_{\bSigma_t} \le \beta_{t, \ell} \ \text{and}\ \|\bmu^* - \bmu_{t,\ell}\|_{\bSigma_{t, \ell}} \le \gamma_{t, \ell} \right\}. \label{event:glsb}
\end{align}
Recall that \begin{align}
    \beta_{t, \ell} = 8\sqrt{d\log \frac{(R + 1)^2\lambda + tA^2}{(R + 1)^2\lambda}\log(4t^2 / \delta)} + 4\sqrt{d}\log(4t^2 / \delta) + 2^\ell{\sqrt{d}} + \sqrt{\lambda}B, \\
    \gamma_{t, \ell} = 8\sqrt{d\log \frac{(R + 1)^2\lambda + tA^2}{(R + 1)^2\lambda}\log(8t^2 T / \delta)} + 4\sqrt{d}\log(8t^2 T / \delta) + \oline{C}_\ell {\sqrt{d}} + \sqrt{\lambda}B. 
\end{align}

\end{definition}

\begin{lemma}[Restatement of Lemma \ref{sec6:confidence}] \label{confidence}
Let $\cE_{1}$ be defined in \eqref{event:glsb}. For any $0 < \delta < 1$, we have $\PP(\cE_{1}) \ge 1 - 3\delta$. 
\end{lemma}

\begin{proof}
    See Appendix \ref{proof:6.4}. 
\end{proof}

\begin{definition}
    For simplicity, we define $\ab_{t, \ell} = \argmax_{\ab \in \cD_t}\max_{\bmu \in \cC_{t, \ell}}\langle \bmu, \ab \rangle$ for each level $\ell$. 
\end{definition}

With this definition, $\ab_t$ can be seen as an action vector randomly chosen from $\ab_{t, \ell}$, $\ell \in [\ell_{\max}]$. In the following part of this section, we show how to derive the instance-independent regret upper bound using this notation. 

\begin{lemma}[Restatement of Lemms \ref{sec6:error:ell}]\label{error:ell}
    Suppose $\cE_1$ occurs. If $f(t) \le \ell^*$, we have $ \langle\ab_t^* -  \ab_t, \bmu^* \rangle \le 2\beta_{t, \ell^*} \|\ab_t\|_{\bSigma_{t}^{-1}} + 2\beta_{t, \ell^*} \|\ab_{t, \ell^*}\|_{\bSigma_t^{-1}}. $ 
\end{lemma}

\begin{lemma}[Restatement of Lemms \ref{sec6:error:ell2}]\label{error:ell2} On event $\cE_1$, if $f(t) = \ell > \ell^*$, we have $\langle \ab_t^* - \ab_t, \bmu^* \rangle \le 2 \gamma_{t, \ell} \|\ab_{t}\|_{\bSigma_{t, \ell}^{-1}}. $
\end{lemma}

\begin{proof}[Proof of Theorem  \ref{thm:regret}]
    Suppose $\cE_1$ occurs. We divide regret into two parts, 
    \begin{align}
        \regret(T) &= \EE\left[\sum_{t = 1}^T \left(\langle \ab_t^*, \bmu^*\rangle - \langle \ab_t, \bmu^* \rangle \right)\right] \notag
        \\&= \underbrace{\EE\left[\sum_{t = 1}^T \mathds{1}(f(t) \le \ell^*)\left(\langle \ab_t^*, \bmu^*\rangle - \langle \ab_t, \bmu^* \rangle \right)\right]}_{I_1}\notag \\&\quad + \sum_{\ell = \ell^* + 1}^{\ell_{\max}}\underbrace{\EE\left[\sum_{t = 1}^T \mathds{1}(f(t)= \ell)\left(\langle \ab_t^*, \bmu^*\rangle - \langle \ab_t, \bmu^* \rangle \right)\right]}_{I_2(\ell)},  \label{I1+I2}
    \end{align}
    where the first equality holds by definition in \eqref{regret}. 
    
    By Lemma \ref{6.8:6.5}, we have
    \begin{align}
        I_1 \le \EE\left[\sum_{t = 1}^T \mathds{1}(f(t) \le \ell^*) \min\left\{2, 2\beta_{t, \ell^*}\|\ab_{t, \ell^*}\|_{\bSigma_t^{-1}} + 2\beta_{t, \ell^*}\|\ab_{t}\|_{\bSigma_t^{-1}} \right\}\right].  \label{I1:4.7}
    \end{align}

    Let $\cF_t$ be the $\sigma$-algebra generated by $\ab_s, r_{s}, \sigma_s, f(s)$ for $s \le t - 1$. 
    Note that 
    \begin{align} 
    	\EE\left[\mathds{1}(f(t)\le \ell^*) \|\ab_{t, \ell^*}\|_{\bSigma_t^{-1}} | \cF_{t}\right] &= \PP(f(t) \le \ell^*) \|\ab_{t, \ell^*}\|_{\bSigma_t^{-1}} \notag \\
    	&\le 2^{\ell^*} \PP(f(t) = \ell^*) \|\ab_{t, \ell^*}\|_{\bSigma_t^{-1}}
    \notag\\&\le 2^{\ell^*} \EE\left[\|\ab_{t}\|_{\bSigma_t^{-1}} | \cF_{t}\right],   \label{I1:4.10}
    \end{align}
    
    where the first equality holds since $\ab_{t, \ell^*}$ and $\bSigma_t$ is deterministic given $\cF_t$, the first inequality holds since $\PP(f(t) = \ell^*) = 2^{-\ell^*}$, the last inequality holds due to  the fact that $\PP(f(t) = \ell^*) \|\ab_{t, \ell^*}\|_{\bSigma_t^{-1}} = \EE\left[\mathds{1}(f(t) = \ell^*)\|\ab_{t}\|_{\bSigma_t^{-1}} | \cF_{t}\right]. $ 
    
    Substituting \eqref{I1:4.10} into \eqref{I1:4.7}, we have 
    \begin{align}
        I_1 \le 2^{\ell^*} \EE\left[\sum_{t = 1}^T 4\min\{\beta_{t, \ell^*} \|\ab_t\|_{\bSigma_t^{-1}}, 2\}\right]
        \le 4\cdot 2^{\ell^*}\EE\underbrace{\left[\sum_{t = 1}^T \min\left\{2, \beta_{T, \ell^*} \|\ab_t\|_{\bSigma_t^{-1}}\right\}\right]}_{I_3} \label{I1:4.12}
    \end{align}

    We split $[T]$ into 2 parts to bound $I_3$. 
    
    Let $\cI_1 = \{t \in [T] | \|\ab_t / \oline{\sigma}_t\|_{\bSigma_t^{-1}} > 1\}, \cI_2 = \{t \in [T] | \|\ab_t / \oline{\sigma}_t\|_{\bSigma_t^{-1}} \le 1\}. $

    \begin{align}
        \sum_{t \in \cI_1} \min\left\{2, \beta_{T, \ell^*}\|\ab_t\|_{\bSigma_t^{-1}}\right\} \le 2 |\cI_1| 
        \le 2\sum_{t = 1}^T \min\left\{1, \|\ab_t / \oline{\sigma}_t\|_{\bSigma_t^{-1}}^2\right\}
        \le 4d \log \frac{(R + 1)^2\lambda + TA^2}{(R + 1)^2 \lambda},  \label{I3:4.15}
    \end{align}
    where the first inequality holds since $\min\left\{2, \beta_{T, \ell^*}\|\ab_t\|_{\bSigma_t^{-1}}\right\} \le 2$, the second inequality follows from the definition of $\cI_1$, the third inequality holds by Lemma \ref{lemma:summation11}. 

    \begin{align}
        \sum_{t \in \cI_2} \min\left\{2, \beta_{T, \ell^*}\|\ab_t\|_{\bSigma_t^{-1}}\right\} &\le \beta_{T, \ell^*} \sqrt{\sum_{t \in \cI_2} \oline{\sigma}_t^2} \cdot \sqrt{\sum_{t \in \cI_2} \min\left\{1, \|\ab_t / \oline{\sigma}_t\|_{\bSigma_t^{-1}}^2\right\}} \notag
        \\&\le \beta_{T, \ell^*} \sqrt{(R + 1)^2T / d + \sum_{t = 1}^T {\sigma}_t^2}\cdot \sqrt{2d \log \frac{(R + 1)^2\lambda + TA^2}{(R + 1)^2 \lambda}},  \label{I3:4.17}
    \end{align}
    where the first inequality follows from Cauchy-Schwarz inequality, the second inequality follows from the definition of $\oline{\sigma}_t$ and Lemma \ref{lemma:summation11}. 

    Substituting \eqref{I3:4.15} and \eqref{I3:4.17} into \eqref{I1:4.12}, we have 

     \begin{align}I_1 = \tilde{O}\left(C^2d\sqrt{\sum_{t = 1}^T \sigma_t^2} + C^2R\sqrt{dT}\right). \label{I1:4.19}\end{align}
    
    By Lemma \ref{6.8:6.6}, 
    \begin{align}
        I_2(\ell) &\le \EE\left[\sum_{t = 1}^T \mathds{1}(f(t) = \ell) \min\left\{2, 2 \gamma_{t, \ell} \|\ab_{t, \ell}\|_{\bSigma_{t, \ell}^{-1}}\right\}\right] \notag
        \\&\le 2 \EE\underbrace{\left[\sum_{t = 1}^T \mathds{1}(f(t) = \ell) \min\left\{1,  \gamma_{t, \ell} \|\ab_{t, \ell}\|_{\bSigma_{t, \ell}^{-1}}\right\}\right]}_{I_4} 
        \label{I2:4.22}
    \end{align}

    Again, we divide $[T]$ into two parts. Let $\cJ_1 =  \{t \in [T] | \|\ab_{t, \ell} / \oline{\sigma}_t\|_{\bSigma_{t, \ell}^{-1}} > 1\}, \cJ_2 =  \{t \in [T] | \|\ab_{t, \ell} / \oline{\sigma}_t\|_{\bSigma_{t, \ell}^{-1}} \le 1\}. $

    \begin{align}
        \sum_{t \in \cJ_1} \mathds{1}(f(t) = \ell) \min\left\{1,  \gamma_{t, \ell} \|\ab_{t, \ell}\|_{\bSigma_{t, \ell}^{-1}}\right\} &\le  \sum_{t \in \cJ_1} \mathds{1}(f(t) = \ell) \cdot 1 \notag
        \\&\le \sum_{t = 1}^T \mathds{1}(f(t) = \ell) \min\left\{1, \|\ab_{t, \ell}\|_{\bSigma_{t, \ell}^{-1}}^2\right\} \notag
        \\&\le 2d \log \frac{(R + 1)^2\lambda + TA^2}{(R + 1)^2\lambda},  \label{I4:4.25}
    \end{align}
    where the second inequality follows from the definition of $\cJ_1$, the second inequality holds due to Lemma \ref{lemma:summation11}. 

    \begin{align}
        &\sum_{t \in \cJ_2}  \mathds{1}(f(t) = \ell) \min\left\{1,  \gamma_{t, \ell} \|\ab_{t, \ell}\|_{\bSigma_{t, \ell}^{-1}}\right\} \notag \\&\le \gamma_{T, \ell}\sqrt{\sum_{t = 1}^T \oline{\sigma}_t^2} \sqrt{\sum_{t \in \cJ_2} \min\left\{1, \|\ab_t / \oline{\sigma}_t\|_{\bSigma_{t, \ell}^{-1}}^2\right\}} \notag
        \\&\le \gamma_{T, \ell} \sqrt{(R + 1)^2T /d + \sum_{t = 1}^T \sigma_t^2}  \sqrt{2d \log \frac{(R + 1)^2\lambda + TA^2}{(R + 1)^2 \lambda}},  \label{I4:4.27}
    \end{align}
    where the first inequality follows from Cauchy-Schwarz inequality, the second inequality follows from the definition of $\oline{\sigma}_t$ and Lemma \ref{lemma:summation11}. 

    Substituting \eqref{I4:4.25} and \eqref{I4:4.27} into \eqref{I2:4.22}, we have 
    \begin{align}
    I_2(\ell) \le 4d \log \frac{(R + 1)^2\lambda + TA^2}{(R + 1)^2\lambda} + 2\gamma_{T, \ell} \sqrt{(R + 1)^2T /d + \sum_{t = 1}^T \sigma_t^2}  \sqrt{2d \log \frac{(R + 1)^2\lambda + TA^2}{(R + 1)^2 \lambda}} \label{I2:6.33}
    \end{align}
    \begin{align} I_2(\ell) = \tilde{O}\left(R\sqrt{Td} + d\sqrt{\sum_{t = 1}^T \sigma_t^2}\right). \label{I2:4.28}\end{align}

    Substituting \eqref{I1:4.19} and \eqref{I2:4.28} into \eqref{I1+I2}, we have $$\regret(T) =  \tilde{O}\left(C^2d\sqrt{\sum_{t = 1}^T \sigma_t^2} + C^2R\sqrt{dT}\right). $$
    
\end{proof}

\subsection{Proof of Theorem \ref{thm:gap dependent regret}}
\begin{proof}[Proof of Theorem \ref{thm:gap dependent regret}]
First we decompose the regret as follows.
    \begin{align}
        \regret(T) &= \EE\left[\sum_{t = 1}^T \left(\langle \ab_t^*, \bmu^*\rangle - \langle \ab_t, \bmu^* \rangle \right)\right] \notag
        \\&\le  \frac{1}{{\Delta}}\underbrace{\EE\left[\sum_{t = 1}^T \mathds{1}(f(t) \le \ell^*)\left(\langle \ab_t^*, \bmu^*\rangle - \langle \ab_t, \bmu^* \rangle \right)^2\right]}_{I_1}\notag \\&\qquad + \sum_{\ell = \ell^* + 1}^{\ell_{\max}}\frac{1}{{\Delta}}\underbrace{\EE\left[\sum_{t = 1}^T \mathds{1}(f(t)= \ell)\left(\langle \ab_t^*, \bmu^*\rangle - \langle \ab_t, \bmu^* \rangle \right)^2\right]}_{I_2(\ell)}, \label{I1+I2:4.30}
    \end{align}
    where the first equality holds due to the definition in \eqref{regret}, the last inequality follows from the fact that either $\langle \ab_t^*, \bmu^*\rangle - \langle \ab_t, \bmu^* \rangle = 0$ or $\oline{\Delta}_T \le \langle \ab_t^*, \bmu^*\rangle - \langle \ab_t, \bmu^* \rangle$. 
    To bound $I_1$, we have
    \begin{align}
        I_1 &\le \EE\left[\sum_{t = 1}^T \mathds{1}(f(t) \le \ell^*) \min\left\{4, \left(2\beta_{t, \ell^*}\|\ab_{t, \ell^*}\|_{\bSigma_t^{-1}} + 2\beta_{t, \ell^*}\|\ab_{t}\|_{\bSigma_t^{-1}}\right)^2 \right\}\right] \notag
        \\&\le 2^{\ell^*}\EE\underbrace{\left[\sum_{t = 1}^T\min\left\{4, 16\beta_{t, \ell^*}^2\|\ab_{t}\|_{\bSigma_t^{-1}}^2 \right\}\right]}_{I_3}, \label{I3:4.32}
    \end{align}
    where the first inequality holds due to Lemma \ref{error:ell} and the second inequality follows from a similar argument as \eqref{I1:4.10}. To further bound $I_3$, we decompose $[T]$ into two non-overlapping sets: $\cI_1 = \{t \in [T] | \|\ab_t / \oline{\sigma}_t\|_{\bSigma_t^{-1}} > 1\}, \cI_2 = \{t \in [T] | \|\ab_t / \oline{\sigma}_t\|_{\bSigma_t^{-1}} \le 1\}$. For $\cI_1$, we have
    \begin{align}
        \sum_{t \in \cI_1}\min\left\{4, 16\beta_{t, \ell^*}^2\|\ab_{t}\|_{\bSigma_t^{-1}}^2 \right\} &\le 4 |\cI_1| \notag \\&\le 4 \sum_{t = 1}^T \min\left\{1, \|\ab_t / \oline{\sigma_t}\|_{\bSigma_{t}^{-1}}^2\right\} \notag
        \\&\le 8d \log \frac{(R + 1)^2 \lambda + TA^2}{(R + 1)^2 \lambda},  \label{I1:6.40}
    \end{align}
    where the third inequality holds due to Lemma \ref{lemma:summation11}. For $\cI_2$, we have
    \begin{align}
        \sum_{t \in \cI_2}\min\left\{4, 9\beta_{t, \ell^*}^2\|\ab_{t}\|_{\bSigma_t^{-1}}^2 \right\} &\le 16\beta_{T, \ell^*}^2 \max_{t \in [T]} \oline{\sigma}_t^2 \sum_{t \in \cI_2} \min\left\{1, \|\ab_t / \oline{\sigma_t}\|_{\bSigma_{t}^{-1}}^2\right\} \notag
        \\&\le 32\beta_{T, \ell^*}^2 (\max_{t \in [T]} {\sigma}_t^2 + (R + 1)^2 / d) d \log \frac{(R + 1)^2\lambda + TA^2}{(R + 1)^2\lambda},  \label{I1:6.42}
    \end{align}
    where the first inequality follows from the definition of $\cI_2$, the second inequality follows from Lemma \ref{lemma:summation11}. 
    
    Substituting \eqref{I1:6.40} and \eqref{I1:6.42} into \eqref{I3:4.32}, we have \begin{align}I_1 = \tilde{O}\left(C^2R^2d + d^2C^2 \max_{t \in[T]} \sigma_t^2\right). \label{I1:6.43} \end{align}
    To bound $I_2(\ell)$, by Lemma \ref{error:ell2}, we have
    \begin{align}
        I_2(\ell) &\le \EE\left[\sum_{t = 1}^T \mathds{1}(f(t) = \ell) \min\left\{4, 4\gamma_{t, \ell}^2\|\ab_t\|_{\bSigma_{t, \ell}^{-1}}^2 \right\}\right] \notag
        \\&\le 4 \EE\underbrace{\left[\sum_{t = 1}^T \mathds{1}(f(t) = \ell) \min\left\{1, \gamma_{t, \ell}^2\|\ab_t\|_{\bSigma_{t, \ell}^{-1}}^2 \right\}\right]}_{I_4}.  \label{I2:6.44}
    \end{align}
    
    We divide $[T]$ into two parts to calculate $I_4$. 
    Let $\cJ_1 = \{t \in [T] | \|\ab_t / \oline{\sigma}_t\|_{\bSigma_{t, \ell}^{-1}} > 1\}, \cJ_2 = \{t \in [T] | \|\ab_t / \oline{\sigma}_t\|_{\bSigma_{t, \ell}^{-1}} \le 1\}. $ For $\cJ_1$, we have
    
    \begin{align}
        \sum_{t \in \cJ_1}  \mathds{1}(f(t) = \ell) \min\left\{1, \gamma_{t, \ell}^2\|\ab_t\|_{\bSigma_{t, \ell}^{-1}}^2 \right\} &\le |\cJ_1| \notag
        \\&\le \sum_{t \in [T]} \min\left\{1, \|\ab_t\|_{\bSigma_{t, \ell}^{-1}}^2\right\} \notag
        \\&\le 2d \log \frac{(R + 1)^2\lambda + TA^2}{(R + 1)^2\lambda}, \label{I4:6.47}
    \end{align}
    where the second inequality follows from the fact that $\cJ_1 \subseteq [T]$, the third inequality holds due to Lemma \ref{lemma:summation11}. For $\cJ_2$, we have
    
    \begin{align}
        \sum_{t \in \cJ_2}  \mathds{1}(f(t) = \ell) \min\left\{1, \gamma_{t, \ell}^2\|\ab_t\|_{\bSigma_{t, \ell}^{-1}}^2 \right\} &\le \gamma_{T, \ell}^2 \max_{t \in [T]} \oline{\sigma}_t^2 \sum_{t \in \cJ_2} \|\ab_t\|_{\bSigma_{t, \ell}^{-1}}^2 \notag
        \\&\le \gamma_{T, \ell}^2 (\max_{t \in [T]} {\sigma}_t^2 + (R + 1)^2 / d) \sum_{t \in [T]} \min\left\{1, \|\ab\|_{\bSigma_{t, \ell}^{-1}}^2\right\} \notag
        \\&\le \gamma_{T, \ell}^2 (\max_{t \in [T]} {\sigma}_t^2 + (R + 1)^2 / d) 2d \log\frac{(R + 1)^2 \lambda + TA^2}{(R + 1)^2\lambda},  \label{I4:6.50}
    \end{align}
    where the second inequality follows from the definition of $\oline{\sigma}_t$ and $\cJ_2$ and the third inequality holds due to Lemma \ref{lemma:summation11}. 
    
    Substituting \eqref{I4:6.47} and \eqref{I4:6.50} into \eqref{I2:6.44}, we have \begin{align}
        I_2(\ell) = \tilde{O}(dR^2 + d^2 \max_{t \in [T]} \sigma_t^2). \label{I2:6.51}
    \end{align}
    Finally, 
    substituting \eqref{I2:6.51} and \eqref{I1:6.43} into \eqref{I1+I2:4.30}, we have $$\regret(T) = \frac{1}{\Delta} \tilde{O}\left(C^2R^2d + d^2C^2 \max_{t \in[T]} \sigma_t^2\right). $$

\end{proof}

\section{Proof of Technical Lemmas in Appendix \ref{section:7}}

\subsection{Proof of Lemma \ref{Bernstein with corruptions(bandit)}} \label{proof:6.1}

\begin{proof}

    Let $\cS(t) = \{1\leq i \le t|y_i \neq y_i^\stoch\}$, 
    $\bbb_t^\stoch = \sum_{i = 1}^t y_i^\stoch \xb_i$ and $\bmu_t^\stoch = \Zb_t^{-1} \bbb_t^\stoch$. 
    By Lemma \ref{Bernstein(bandit)}, we have that with probability at least $1 - \delta$, $\|\bmu_t^\stoch - \bmu^*\|_{\mathbf{Z}_t} \le \beta_t + \sqrt{\lambda} \|\bmu^*\|_2$ holds for all $t \ge 1$. 

    Also, we have \begin{align*}
        \|\bmu_t - \bmu_t^\stoch\|_{\Zb_t} &= \|\Zb_t^{-1}(\bbb_t - \bbb_t^\stoch)\|_{\Zb_t}
        \\&\le \sum_{i=1}^t \|\Zb_t^{-1} (y_i^\stoch - y_i)\xb_i\|_{\Zb_t}
        \\&\le \sum_{i=1}^t |y_i^\stoch - y_i| \cdot \|\xb_i\|_{\Zb_t^{-1}}
        \\&\le C(t).
    \end{align*}
    where the first inequality holds due to the triangle inequality and the last inequality holds due to $\|\xb_i\|_{\Zb_t^{-1}}\leq 1$. 

    Hence, we can obtain $$\|\bmu_t - \bmu^*\|_{\Zb_t} \le \|\bmu_t^\stoch - \bmu^*\|_{\mathbf{Z}_t} + \|\bmu_t - \bmu_t^\stoch\|_{\Zb_t} \le \beta_t + C(t) + \sqrt{\lambda} \|\bmu^*\|_2. $$
\end{proof}

\subsection{Proof of Lemma \ref{sb-corrupt-bandit}} \label{proof:6.2}

\begin{lemma}[Lemma 3.3, \citealt{lykouris2018stochastic}] \label{sb-corrupt-bandit:helper}  Define the corruption level for a level $\ell$: $$\text{Corruption}_{t, \ell} = \sum_{i = 1}^t \frac{\mathds{1}(f(i) = \ell)}{R + 1} \cdot \sup_{\ab \in \cD_i} \left|r_i(\ab) - r_i'(\ab)\right|. $$Then we have for all $\ell \ge \ell^*$, with probability at least $1 - \delta$: 
    $$\text{Corruption}_{t, \ell} \le \log(1 / \delta) + 3, \qquad \forall t \ge 1. $$

\end{lemma}

\begin{proof}[Proof of Lemma \ref{sb-corrupt-bandit}]
    Applying Lemma \ref{sb-corrupt-bandit:helper}, we have for all $\ell \ge \ell^*$, with probability at least $1 - \delta / (2 \ell^2)$: 
    $\text{Corruption}_{t, \ell} \le \log(2 \ell^2 / \delta) + 3, \forall t \ge 1. $
    
    Using a union bound over all $\ell \ge \ell^*$, we can prove the lemma.
\end{proof}

\subsection{Proof of Lemma \ref{confidence}} \label{proof:6.4}

To prove the lemma, we first define the following two events: \begin{align}
    \cE_{2} &:= \left\{\forall \ell \ge \ell^*\ \text{and}\ t\ge 1, \|\bmu^* - \bmu_t\|_{\bSigma_t} \le \beta_{t, \ell} \right\}\label{event:gl} \\
\cE_{3} &:= \left\{\forall \ell \ge \ell^*\ \text{and}\ t\ge 1, \|\bmu^* - \bmu_{t,\ell}\|_{\bSigma_{t, \ell}} \le \gamma_{t, \ell} \right\}\label{event:sb}
\end{align}

\begin{lemma}\label{confidence:gl}
    Let $\cE_{2}$ be defined in \eqref{event:gl}. For any $0 < \delta < 1$, we have $\PP(\cE_{2}) \ge 1 - \delta$. 
\end{lemma}
\begin{proof}
    Applying Lemma \ref{Bernstein with corruptions(bandit)}, we have that $$\|\bmu_t - \bmu^*\|_{\mathbf{\Sigma}_t} \le 8\sqrt{d\log \frac{(R + 1)^2\lambda + tA^2}{(R + 1)^2\lambda}\log(4t^2 / \delta)} + 4\sqrt{d}\log(4t^2 / \delta) + C{\sqrt{d}} + \sqrt{\lambda} \|\bmu^*\|_2$$ for all $t \ge 1$ with probability at least $1 - \delta$. Note that $2^{\ell} \ge C$ for all $\ell \ge \ell^*$, which indicates that $\cE_{2}$ occurs with probability at least $1 - \delta$. 
\end{proof}

\begin{lemma} \label{confidence:sb}
    Let $\cE_{3}$ be defined in \eqref{event:sb}. For any $0 < \delta < 1$, we have $\PP(\cE_{3}) \ge 1 - 2\delta$. 
\end{lemma}

\begin{proof}
    Applying Lemma \ref{Bernstein with corruptions(bandit)}, we have that $\|\bmu_{t, \ell} - \bmu^*\|_{\mathbf{\Sigma}_t} \le 8\sqrt{d\log \frac{(R + 1)^2\lambda + tA^2}{(R + 1)^2\lambda}\log(4t^2 T / \delta)} + 4\sqrt{d}\log(4t^2 T / \delta) + \text{Corruption}_{t, \ell}{\sqrt{d}} + \sqrt{\lambda} \|\bmu^*\|_2$ for all $t \ge 1$ with probability at least $1 - \delta / \ell$. Here we use the fact that $\ell \le T$.
    Applying Lemma \ref{sb-corrupt-bandit} and a union bound, we have $\cE_3$ occurs with probability at least $1 - 2\delta$. 
\end{proof}

\begin{proof}[Proof of Lemma \ref{confidence}]
    This lemma can be proved by a union bound on $\cE_2$ and $\cE_3$ with Lemmas \ref{confidence:gl} and \ref{confidence:sb}. 
\end{proof}

\subsection{Proof of Lemma \ref{error:ell}}

\begin{proof}
For simplicity, let 
$\cA_{t,\ell} = \left \{\bmu|\|\bmu - \bmu_t\|_{\bSigma_{t}} \le \beta_{t, \ell}\right\}, \cB_{t,\ell}= \left\{\bmu|\|\bmu - \bmu_{t, \ell}\|_{\bSigma_{t, \ell}} \le \gamma_{t, \ell}\right\}. $
    Let $\bmu_t^m = \argmax_{\bmu \in \cC_{t, f(t)}} \langle \ab_t, \bmu\rangle. $ Then we have \begin{align}
        \langle \ab_t, \bmu^* \rangle  &\ge \langle \ab_t, \bmu_t \rangle - \beta_{t, \ell^*} \|\ab_t\|_{\bSigma_{t}^{-1}} \notag
        \\&\ge \langle \ab_t, \bmu_t^m \rangle - 2\beta_{t, \ell^*} \|\ab_t\|_{\bSigma_{t}^{-1}} \notag\\
        &\ge \langle \ab_{t, \ell^*}, \bmu_t^m\rangle - 2\beta_{t, \ell^*} \|\ab_t\|_{\bSigma_{t}^{-1}} \notag
        \\&\ge \langle \ab_{t, \ell^*}, \bmu_t\rangle - \beta_{t, \ell^*} \|\ab_{t, \ell^*}\|_{\bSigma_t^{-1}} - 2\beta_{t, \ell^*} \|\ab_t\|_{\bSigma_{t}^{-1}} \notag
         \\&\ge \max_{\bmu \in \cA_{t, \ell^*}} \langle \ab_{t, \ell^*}, \bmu \rangle -2 \beta_{t, \ell^*} \|\ab_{t, \ell^*}\|_{\bSigma_t^{-1}} - 2\beta_{t, \ell^*} \|\ab_t\|_{\bSigma_{t}^{-1}}, \notag
        \\&\ge \max_{\bmu \in \cC_{t, \ell^*}} \langle \ab_{t, \ell^*}, \bmu \rangle -2 \beta_{t, \ell^*} \|\ab_{t, \ell^*}\|_{\bSigma_t^{-1}} - 2\beta_{t, \ell^*} \|\ab_t\|_{\bSigma_{t}^{-1}},  \label{6.8:6.5}
    \end{align}
    where the first inequality holds since $\bmu^* \in \cC_{t, \ell^*} \subseteq \cA_{t, l^*}$, the second inequality holds since $\bmu_t^m \in \cC_{t, f(t)} \subseteq \cA_{t, l^*}$,  the third inequality holds by the definition of $\ab_t$ and $\bmu_t^m$, the fourth inequality holds since $\bmu_t^m \in \cA_{t, \ell^*}$, the fifth inequality holds since $\bmu_t \in \cA_{t, \ell^*}$, the last one holds since $\cC_{t, \ell^*} \subseteq \cA_{t, \ell^*}$. 
    By the definition of $\cE_1$ and $\ab_{t, \ell^*}$, we have \begin{align}
        \max_{\bmu \in \cC_{t, \ell^*}} \langle \ab_{t, \ell^*}, \bmu \rangle =  \max_{\ab \in \cD_t} \max_{\bmu \in \cC_{t, \ell^*}} \langle \ab, \bmu \rangle \ge \max_{\ab \in \cD_t} \langle \ab, \bmu^*\rangle = \langle \ab_t^*, \bmu^* \rangle.  \label{6.8:6.6}
    \end{align}
    Combining \eqref{6.8:6.5} with \eqref{6.8:6.6}, we have $\langle\ab_t^* -  \ab_t, \bmu^* \rangle \le 2\beta_{t, \ell^*} \|\ab_t\|_{\bSigma_{t}^{-1}} + 2\beta_{t, \ell^*} \|\ab_{t, \ell^*}\|_{\bSigma_t^{-1}}. $
\end{proof}

\subsection{Proof of Lemma \ref{error:ell2}}

\begin{proof}
We have
    \begin{align}
        \langle \ab_t^* - \ab_t, \bmu^* \rangle &\le
        \max_{\bmu \in \cC_{t, \ell}} \langle \ab_{t}, \bmu \rangle - \langle \ab_t, \bmu^* \rangle \notag
        \le 2 \gamma_{t, \ell} \|\ab_{t}\|_{\bSigma_{t, \ell}^{-1}}, 
    \end{align}
    where the first inequality follows from the fact that $\bmu^* \in \cC_{t, \ell}$ and the definition of $\ab_t$, the second inequality holds since $\bmu^* \in \cC_{t, \ell}$ on the event $\cE_1$.  
\end{proof}

\section{Auxiliary Lemmas}

\begin{lemma}[Theorem 4.1, \citealt{Zhou2020NearlyMO}]\label{Bernstein(bandit)}
    Let $\{\mcal{G}_t\}_{t = 1}^\infty$ be a filtration, $\{\xb_t, \eta_t\}_{t \ge 1}$ a stochastic process so that $\mathbf{x}_t \in \mbb{R}^d$ is $\mcal{G}_t$-measurable and $\eta_t \in \mbb{R}$ is $\mcal{G}_{t + 1}$-measurable. Fix $R, L, \sigma, \lambda > 0, \ \bmu^* \in \mbb{R}^d$. For $t \ge 1$ let ${y}_t = \langle \mathbf{\bmu}^*, {\mathbf{x}}_t \rangle + {\eta}_t$ and suppose that ${\eta}_t, {\mathbf{x}}_t$ also satisfy $$|{\eta}_t| \le R, \mbb{E}[{\eta}_t|\mcal{G}_t] = 0, \mbb{E}[{\eta}_t^2|\mcal{G}_t] \le \sigma^2, \|{\mathbf{x}}_t\|_2 \le L. $$

    Then, for any $0 < \delta < 1$, with probability at least $1 - \delta$ we have $\forall t > 0$, \begin{align*} \|\bmu_t - \bmu^*\|_{\mathbf{Z}_t} \le \beta_t + \sqrt{\lambda} \|\bmu^*\|_2, \end{align*} where for $t \ge 1$, $\bmu_t = \mathbf{Z}_t^{-1}\mathbf{b}_t$, $\mathbf{Z}_t = \lambda \mathbf{I} + \sum_{i = 1}^t {\mathbf{x}}_i {\mathbf{x}}_i^\top$, $\mathbf{b}_t = \sum_{i = 1}^t {y}_i{\mathbf{x}}_i$, and $$\beta_t = 8\sigma\sqrt{d\log \frac{d\lambda + tL^2}{d\lambda}\log(4t^2 / \delta)} + 4R\log(4t^2 / \delta). $$
\end{lemma}

\begin{lemma}[Lemma 11, \citealt{AbbasiYadkori2011ImprovedAF}] \label{lemma:summation11}
    For any $\lambda > 0$ and sequence $\{\xb_t\}_{t = 1}^T \subset \mbb{R}^d$ for $t\in 0 \cup [T]$, define $\Zb_t = \lambda \mathbf{I} + \sum_{i = 1}^t \xb_i \xb_i^\top$. Then, provided that $\|\xb_t\|_2 \le L$ holds for all $t \in [T]$, we have $$\sum_{t = 1}^T \min\{1, \|\xb_t\|_{\Zb_{t - 1}^{-1}}^2\}\le 2d \log \frac{d\lambda + TL^2}{d\lambda}. $$
\end{lemma}



\bibliographystyle{ims}

\bibliography{ref}

\end{document}